\documentclass[doublecolumn]{IEEEtran}

\usepackage{graphicx}
\usepackage[percent]{overpic}
\usepackage{subfig}
\usepackage{amsmath}
\usepackage{amsthm} 
\usepackage{amssymb}
\usepackage[font=footnotesize]{caption} 
\usepackage{subfig} 
\usepackage[noadjust]{cite} 
\usepackage{color}
\usepackage{algorithm} 
\usepackage{algpseudocode} 
\usepackage{enumerate}
\usepackage[hidelinks,colorlinks=false]{hyperref}
\usepackage[left=0.75in, right=0.75in, top=0.75in, bottom=0.75in]{geometry}
\usepackage{authblk} 
\usepackage{arydshln} 

\usepackage{bm}
\usepackage{tikz}
\usetikzlibrary{calc} 
\usetikzlibrary{shapes} 
\usetikzlibrary{chains}
\usetikzlibrary{fit}
\usetikzlibrary{arrows}
\usetikzlibrary{decorations.text} 
\usetikzlibrary{decorations.markings}
\usetikzlibrary{decorations.pathmorphing} 
\usetikzlibrary{shadows}
\usetikzlibrary{patterns}
\usetikzlibrary{matrix}
\usepackage{pgfplots}
\usepackage[europeanresistors]{circuitikz}
\usepackage[outline]{contour} 
\contourlength{1.5pt}

\newtheorem{theorem}{Theorem}

\graphicspath{{figures/}}

\newcommand{\red}{\textcolor{red}}
\usepackage{graphicx}
\usepackage{caption}
\usepackage{verbatim}

\begin{document}
\title{LeTac-MPC: Learning Model Predictive Control for Tactile-reactive Grasping}
\author{
Zhengtong Xu, Yu She$^{*}$
\thanks{$^{*}$Address all correspondence to this author.}
\thanks{Zhengtong Xu and Yu She are with the School of Industrial Engineering, Purdue University, West Lafayette, USA  (E-mail: \{xu1703, shey\}@purdue.edu).}
}

\let\oldtwocolumn\twocolumn
\renewcommand\twocolumn[1][]{%
    \oldtwocolumn[{#1}{
    \begin{center}
    \includegraphics[trim=100 0 0 0,clip,width=0.85\textwidth]{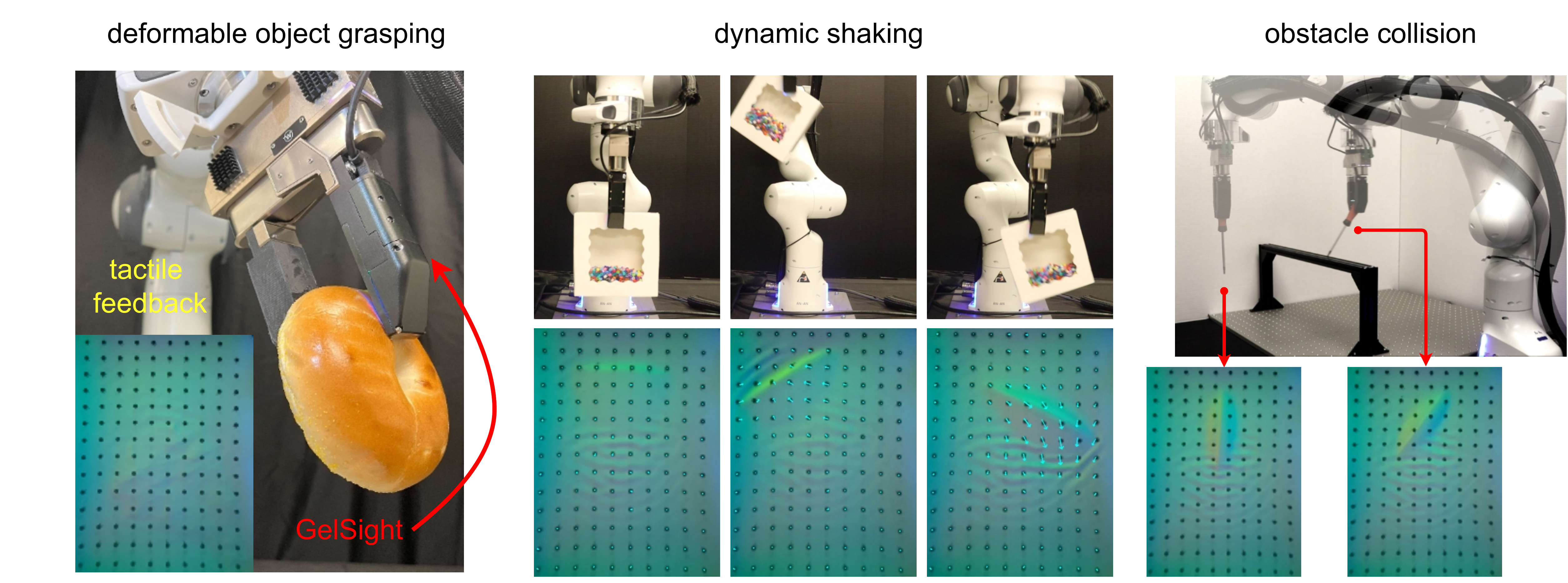}
           \captionof{figure}{LeTac-MPC is capable of reactive grasping for objects with varying physical properties, shapes, sizes, and surface textures. It enables robust grasping in dynamic manipulation scenarios, adjusting the gripper width to provide the appropriate grasping force based on tactile feedback. This ensures stability without damaging the grasped object. Our demonstrations show the robot performing various tasks using LeTac-MPC. Left: The robot grasps a deformable object. Middle and right: LeTac-MPC allows the gripper to maintain a stable grasping under dynamic shaking and unpredictable external collisions, respectively. }
           \label{fig:firstPage}
        \end{center}
    }]
}

\maketitle

\begin{abstract}
Grasping is a crucial task in robotics, necessitating tactile feedback and reactive grasping adjustments for robust grasping of objects under various conditions and with differing physical properties. In this paper, we introduce LeTac-MPC, a learning-based model predictive control (MPC) for tactile-reactive grasping. Our approach enables the gripper to grasp objects with different physical properties on dynamic and force-interactive tasks. We utilize a vision-based tactile sensor, GelSight \cite{yuan2017gelsight}, which is capable of perceiving high-resolution tactile feedback that contains information on the physical properties and states of the grasped object. LeTac-MPC incorporates a differentiable MPC layer designed to model the embeddings extracted by a neural network (NN) from tactile feedback. This design facilitates convergent and robust grasping control at a frequency of 25 Hz. We propose a fully automated data collection pipeline and collect a dataset only using standardized blocks with different physical properties. However, our trained controller can generalize to daily objects with different sizes, shapes, materials, and textures. The experimental results demonstrate the effectiveness and robustness of the proposed approach. We compare LeTac-MPC with two purely model-based tactile-reactive controllers (MPC and PD) and open-loop grasping. Our results show that LeTac-MPC has optimal performance in dynamic and force-interactive tasks and optimal generalizability. We release our code and dataset at https://github.com/ZhengtongXu/LeTac-MPC.
\end{abstract}

\section{Introduction}\label{sec:intro}
Grasping is a reactive task in which humans estimate the states and physical properties of the grasped object by tactile feedback of the fingers and dynamically adjust finger behavior. This allows for robust grasping regardless of the object's stiffness, whether it needs to be stationary or dynamically moved, and whether it is subjected to external forces. An ideal grasping algorithm for robots should be able to do the same. In this paper, we study the task of robotic tactile-reactive grasping. We specifically focus on a vision-based tactile sensor, GelSight \cite{yuan2017gelsight}, due to its advantages of perceiving high-resolution tactile feedback that contains rich information on the physical properties and states of the grasped object.

In recent years, many works demonstrate that tactile sensors can perceive various information of objects, such as texture \cite{luo2018vitac,ward2020neurotac,kaboli2018robust}, material \cite{omarali2022tactile,luo2017robotic}, shape \cite{yin2022multimodal,murali2022deep}, and hardness \cite{yuan2017shape}. However, utilizing tactile feedback to design robotic tactile-reactive grasping controllers faces several challenges.

1)  \textbf{Complex object physical properties: } Objects may have different physical properties, which poses great challenges to tactile-based perception. Existing tactile perception methods often assume that grasped objects are rigid \cite{dong2019maintaining,hogan2020tactile}, and sensors, such as vision-based tactile sensors \cite{lin2022dtact,zhang2022deltact,donlon2018gelslim,lambeta2020digit}, may not be sensitive if the grasped object is softer than the sensor elastomer \cite{yuan2017gelsight}. As a result, it can be difficult for controllers to get high-quality and stable feedback signals using model-based extraction methods when handling objects with different physical properties, and it can be challenging to generalize the same set of controller parameters to different objects, as their physical properties and dynamics may differ.

2) \textbf{Dynamic and force-interactive tasks: } Statically grasping an object is often insufficient for dynamic and force-interactive tasks. For example, when robots use the gripper to pick fruit, there is a force interaction between the gripper, the fruit, and the branches. Similarly, in dynamic pick-and-place tasks, the grasped object may experience inertia force and collide with unexpected obstacles during transportation. In both of these scenarios, if a gripper is not responsive enough to state changes, external disturbances, and force interactions, it can result in the grasped object falling from the gripper. In addition, if a larger grasping force is used for a stronger grasp, delicate objects such as fruits, eggs, bread, and crackers can be destroyed. Therefore, it is challenging to design a reactive controller that can grasp with appropriate force according to the physical properties and states of the object.

3) \textbf{High-resolution tactile feedback integration:} 
Another challenge is the integration of high-resolution tactile feedback, such as the vision-based tactile sensor GelSight \cite{yuan2017gelsight}, which contains rich information, into a real-time control loop for tactile-reactive grasping. Traditional control methods are based on low-dimensional feedback signals, which require the extraction of required information from tactile images for reduction of dimensionality \cite{she2021cable, shirai2023tactile,lambeta2020digit}. However, many feature extraction methods make assumptions about the physical properties and shape of the object \cite{dong2019maintaining,hogan2020tactile,she2021cable}, making it difficult to generalize to objects with different physical properties and shapes. On the other hand, robot learning methods are based on high-dimensional observations such as visual and tactile images. However, these methods often overlook crucial aspects of the control process, such as convergence, response speed, control frequency, and constraints. As a result, they are more suitable for generating open-loop or low-frequency actions, but not well-suited for tasks that demand rapid and reactive responses.

To address these challenges, we propose LeTac-MPC, a learning-based model predictive control (MPC) for tactile-reactive grasping. 
 The novelties of the proposed approach are as follows.

1) We design a neural network (NN) architecture with a differentiable MPC layer as the output layer. The MPC layer is modeled according to the control objectives to rationalize the embeddings extracted by the NN from the tactile feedback. The combination of NN and MPC layer provides the advantages of strong generalizability and ease of implementation as a real-time controller. We implement the trained NN and MPC layer in a model predictive controller that is convergent and can run at a frequency of 25~Hz. This controller enables robust reactive grasping of various daily objects in dynamic and force-interactive tasks.

2) To train the model, we use four standardized blocks of different materials with different physical properties to collect a dataset and propose a fully automated data collection pipeline. Finally, we show that despite our data collection with standardized blocks in terms of sizes, shapes, and materials, and with no textures on the block surface, our trained controller can generalize to various daily objects with different sizes, shapes, materials, and textures. This shows the strong generalizability of LeTac-MPC.

3) We also compare two purely model-based tactile-reactive controllers (MPC and PD) as baseline methods. Through experimental comparisons between these two model-based controllers, LeTac-MPC, and open-loop grasping, we show that LeTac-MPC has optimal performance in dynamic and force-interactive tasks and optimal generalizability.
\section{Related work}\label{sec:related}

This section provides a summary of previous research related to exploiting tactile sensing for control and learning.

\subsection{Tactile Control}

Previous research has shown that tactile feedback provided by vision-based tactile sensors can be used to design manipulation controllers for rigid objects \cite{hogan2020tactile,kim2021active,shirai2023tactile}, cables \cite{she2021cable}, and clothes \cite{sunil2022visuotactile}. These methods extract state estimations of the manipulated object, such as pose and contact line with the environment, from tactile feedback, and then design controllers. However, these approaches are not generalizable to objects with different physical properties due to limitations in their model-based state extraction methods.  The work presented in \cite{lloyd2021goal} proposes a tactile control framework for robotic pushing that uses tactile feedback to estimate and predict state information. In \cite{zheng2022autonomous}, a learning-based method is proposed for flipping pages via tactile sensing. However, these two works do not fully use the rich information about the object's physical properties that can be obtained from tactile feedback, which is critical for enhancing generalization to different objects.
 
Recently, the work in \cite{oller2022manipulation} proposes a method to learn soft tactile sensor membrane dynamics. However, this method does not consider the dynamics of the manipulated object, which is important for manipulating deformable objects. The work in \cite{tian2019manipulation} proposes a learning-based tactile MPC framework. However, this framework is not suitable for real-time control tasks because it uses a video prediction model, resulting in high computational cost. As mentioned in the paper, during the experimental testing, this MPC only runs at around 1~Hz.

\subsection{Tactile Grasping}

Multiple works exploit tactile feedback and NN to predict grasp success and stability \cite{calandra2017feeling,si2022grasp,kanitkar2022poseit,yan2022robotic}. Based on the predictive ability of NN, the works in \cite{calandra2018more,hogan2018tactile,han2021learning,feng2020center} sample potential grasping actions and select the action with the lowest cost as the optimal grasping. However, these methods do not focus primarily on real-time control. They do not consider important aspects such as convergence, response speed, control frequency, and constraints in the control process. Consequently, they are better suited for low-frequency action generation rather than real-time control. For example, in \cite{han2021learning}, the computation times of different NN models for one trial are both more than 0.28~s. As a result, these methods cannot respond fast and smooth enough when the robot performs dynamic and force-interactive tasks. The works in \cite{james2020slip,dong2019maintaining} maintain stable grasping by slip detection. Therefore, the accuracy of slip detection determines the performance of the method. However, in practical situations, slip detection may be insufficiently sensitive due to the different shapes and physical properties of grasped objects. 

In \cite{kolamuri2021improving,she2021cable}, grasping controllers with model-based tactile feature extraction are proposed. However, these methods cannot be widely applicable to objects with different shapes and physical properties since some feature extraction methods only work on specific types of object. For example, GelSight, as used in \cite{kolamuri2021improving,she2021cable}, cannot obtain obvious features for objects softer than the sensor gel \cite{yuan2017gelsight}. These features are extremely difficult to process and extract using model-based methods. Instead, existing work on tactile perception shows that NN can effectively extract these features \cite{yuan2017shape}. Therefore, it is crucial to expand learning-based tactile perception to develop a learning-based tactile control method that can generalize to objects with different shapes, textures, and physical properties.

\subsection{Learning Physical Properties for Manipulation}

Learning-based methods that incorporate tactile and visual feedback can achieve dynamic manipulation without prior knowledge of the physical properties of the manipulated object \cite{wang2020swingbot,zeng2020tossingbot,chi2022iterative}. However, these methods are not suitable for tasks that require reactive behavior. In contrast, many manipulation tasks require real-time sensing and real-time control. In \cite{li20223d,driess2022learning}, the authors propose methods that utilize neural radiance fields (NeRF) for manipulating objects with complex properties. However, these methods are not suitable for tasks demanding real-time control and fast response due to the relatively high computational costs. In \cite{yu2022global}, a model learning method for elastic deformable linear objects is proposed that enables real-time cable manipulation. The input of its model is the cable state extracted from the raw image, which requires markers on the cable for tracking. However, in practice, extracting cable states from images without markers is challenging. In summary, achieving real-time, learning-based robot control directly using high-dimensional images as input remains a significant challenge.

The work in \cite{bi2021zero} proposes a learning-based real-time tactile control method to swing up rigid poles with different physical properties. However, there is no research on real-time tactile controllers for soft and deformable objects.

\section{Approach}

This section outlines our proposed learning-based MPC for tactile-reactive grasping. Initially, we design an NN architecture that includes a differentiable MPC layer (Sections \ref{subsec:encoding} and \ref{subsec:mpc_layer}) to extract and represent the diverse and complex physical properties of grasped objects. Then, we introduce the details of training our proposed model (Section \ref{sec:training}). Furthermore, we describe our automated data collection pipeline in Section \ref{subsec:data_collection}. Finally, we present the implementation of the NN as a real-time controller in Section \ref{subsec:imp}.

By our method, the gripper can successfully grasp objects with different physical properties and shapes while applying the appropriate force to avoid damage to fragile and delicate items. Our method can also adjust the grasping in real-time based on tactile feedback to improve the grasping robustness when performing dynamic and force-interactive tasks.

\subsection{Tactile Information Encoding}\label{subsec:encoding}

 We use GelSight \cite{yuan2017gelsight}, a vision-based tactile sensor that provides high-resolution images of the contact surface geometry and strain field, as shown in Fig.~\ref{fig:firstPage}. To extract states and physical properties of grasped objects, we use a convolutional neural network (CNN) with a multi-layer perceptron (MLP) to encode a tactile image into a low-dimensional embedding $\mathbf{f}\in\mathbb{R}^M$, where $M$ is the dimension of the embedding. The works on tactile perception in \cite{yuan2017shape,yuan2018active} demonstrate effective feature extraction from tactile images using pre-trained visual models. In our implementation, we use a pre-trained ResNet-152 \cite{he2016deep} as the CNN architecture and replace the last layer with a two-layer MLP with ReLU activation. 

\begin{figure*}[ht!]
\centering
\begin{overpic}[trim=0 0 0 0,clip, width=0.95\textwidth]{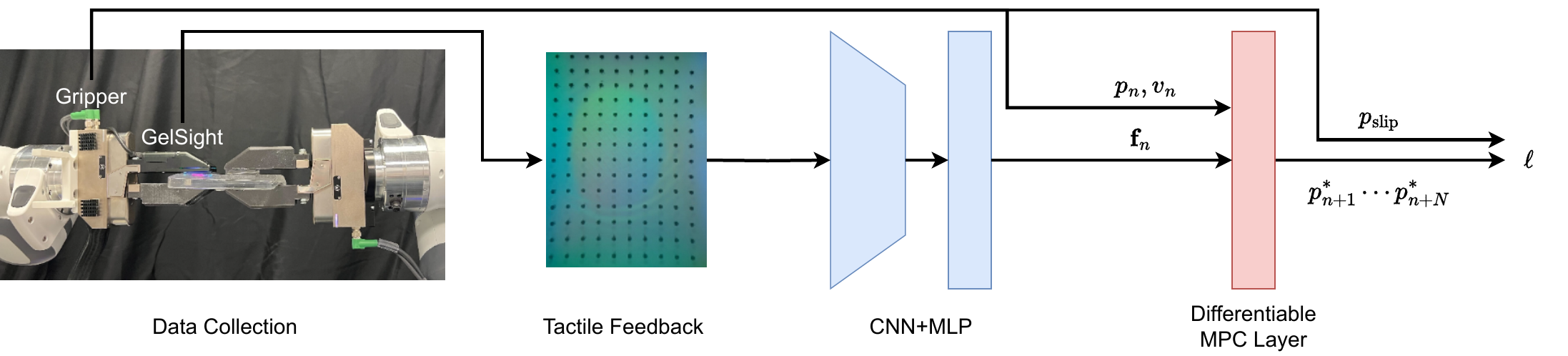}
\end{overpic}
\caption{LeTac-MPC network model. We use the raw image of tactile feedback as the model input. $\ell$ is our proposed loss function. }
\label{fig:model_pipline}
\end{figure*}

\subsection{Differentiable MPC Layer}\label{subsec:mpc_layer}

{In this section, we introduce a differentiable optimization layer that utilizes the extracted embeddings for controller design, as depicted in Fig.~\ref{fig:model_pipline}. This layer formulates an MPC problem, which we henceforth refer to as a differentiable MPC layer.}

Differentiable optimization layers \cite{amos2017optnet} function similarly to standard layers in neural networks, capable of performing both forward passes and backpropagation in batch form. During the forward pass, the layer takes in variables that define an optimization problem and outputs results related to the solution of that problem. Consequently, the forward pass of this layer involves solving optimization problems in batch form. Backpropagation in the differentiable optimization layer involves updating the parameters of the optimization problem based on the calculated gradients.

In this paper, we consider the two-finger gripper with one degree of freedom linear motion. This kind of gripper is widely used because of its practicality. Denoting the sampling time interval as $\Delta t$, we can first formulate the gripper motion model as:

\begin{align}
\label{eq:gripper_model}
\left[\begin{array}{c}
    p_{n+1}\\
    v_{n+1}
\end{array}\right] 
&= 
\mathbf{A}_g
\left[\begin{array}{c}
    p_{n}\\
    v_{n}
\end{array}\right]+
\mathbf{B}_g
a_n,\\
\text{where}~\mathbf{A}_g&= 
\left[\begin{array}{cc}
    1 & \Delta t\\
    0 & 1 
\end{array}\right]\in\mathbb{R}^{2 \times 2},
\mathbf{B}_g=
\left[\begin{array}{c}
    \frac{1}{2}\Delta t^2\\
    \Delta t  
\end{array}\right]\in\mathbb{R}^2.\notag
\end{align}
{Scalars $p,v$, and $a$ are the position, velocity, and acceleration of the motion of gripper fingers. Note that $p,v$, and $a$ do not refer to the motion of the gripper itself in Cartesian space but rather denote the one-degree-of-freedom motion of the gripper fingers.} The right subscript $n$ represents the $n$-th time step. 

In order to utilize the embeddings $\mathbf{f}\in\mathbb{R}^M$ from a NN to design a controller, we propose a differentiable MPC layer that takes into account the fact that the features of tactile images change as the gripper width changes when grasping an object. We formulate the relationship between the tactile embeddings and gripper states as a linear model 
$$ \mathbf{f}_{n+1} = \mathbf{f}_{n} +  {\mathbf{A}}_f v_n, \notag$$
where  $\mathbf{A}_f \in \mathbb{R}^{M}$  would be learned. This modeling approach does not lose its applicability to complex dynamics because the tactile embeddings are actually generated by an NN, which has strong nonlinear fitting capabilities. Moreover, for MPC, we do not expect the model to maintain high accuracy over long timescales because MPC operates at a high frequency with receding horizon control. Therefore, as long as the model demonstrates adequate representational capability in a short-term local area, it is sufficient for MPC performance. The combination of the NN and the MPC layer provides the advantages of strong generalizability and ease of implementation as a real-time controller. The forward pass of the MPC layer involves solving the following optimization problem:

\begin{align}
\label{eq:mpc_layer}
        \mathbf{a}_n^*=&\arg \min _{\mathbf{a}_n}          P(\mathbf{f}_{n+N}^T \mathbf{Q}_f \mathbf{f}_{n+N}+
        Q_v v_{n+N}^2)\notag\\
    &+\sum_{k=n}^{n+N-1}
        \mathbf{f}_{k}^T \mathbf{Q}_f \mathbf{f}_{k}+
        Q_v v_{k}^2+Q_a a_{k}^2,\\
        \text{subject to}&\notag\\
\left[\begin{array}{c}
    \mathbf{f}_{n+1}\\
    p_{n+1}\\
    v_{n+1}
\end{array}\right] 
&=\left[\begin{array}{ccc}
    \mathbf{I}^{M\times M}&\bar{\mathbf{A}}_f\\
    \mathbf{0}^{2\times M}&\mathbf{A}_g
\end{array}\right]
\left[\begin{array}{c}
    \mathbf{f}_{n}\\
    p_{n}\\
    v_{n}
\end{array}\right]+
\left[\begin{array}{c}
\mathbf{0}^{M\times1}\\
\mathbf{B}_g
\end{array}\right]
a_n,\label{eq:mpc_layer_tran}\\
\text{where}~\bar{\mathbf{A}}_f &= \left[\begin{array}{cc}\mathbf{0}^{M\times 1}&\mathbf{A}_f\end{array}\right]\in \mathbb{R}^{M \times 2}.\notag
\end{align}
{$n$ denotes the current time step, and we can define the initial time step as $n=0$. As time progresses, there is no upper limit to $n$. Therefore, $n = 0,1,2,3,\ldots$.}  $\mathbf{a}_n = [a_n,a_{n+1},\ldots{},a_{n+N-1}]^T\in\mathbb{R}^{N}$ is the acceleration sequence, $N$ is the prediction horizon, the vector $\mathbf{A}_f \in \mathbb{R}^{M}$ is part of the state transition matrix, $\mathbf{Q}_f\in \mathbb{R}^{M \times M}$ is the weight matrix for $\mathbf{f}_n$, and $Q_v$ and $Q_a$ are the weight coefficients for $v_n$ and $a_n$, respectively. The scalar $P$ is used to amplify the terminal cost to speed up convergence. {Note that  the ``$n$" in $v_n, v_{n+1}$ represents the current time step $n$, while the ``$N$" in $v_{n+N}$ represents the MPC horizon length $N$. } Importantly, we design $\mathbf{f}_{n+1}$ is only directly related to $\mathbf{f}_{n}$ and $v_n$. This design helps the model avoid the influence of the different sizes of the grasped objects. Therefore, the trained model can be generalized to objects with different sizes naturally.

Upon solving the optimization problem \eqref{eq:mpc_layer} to obtain $\mathbf{a}_n^* = [a_n^*,a_{n+1}^*,\ldots,a_{n+N-1}^*]^T\in\mathbb{R}^{N}$, we can generate the entire trajectory of the gripper motion output by the MPC controller, including $p^*,v^*$ and $a^*$, using equation \eqref{eq:gripper_model}. Ultimately, the final output of the MPC layer consists of the position sequence $[p^*_{n+1},\ldots,p^*_{n+N}]^T\in\mathbb{R}^{N}$, as shown in Fig.~\ref{fig:model_pipline}. The rationale behind this design is that most grippers, such as the WSG gripper used in this paper, have a low-level tracking controller capable of effectively tracking high-level position commands. By outputting the position sequence of the trajectory generated by the MPC layer and sending it to the low-level controller, we can approximate the tracking of the entire trajectory, including $p^*,v^*$ and $a^*$, in a simple yet effective manner. For more details, see the experiments described in Section \ref{sec:exp}.

\subsection{Training}\label{sec:training}

For this MPC layer, we choose $\mathbf{A}_f \in \mathbb{R}^{M}$ and $\mathbf{Q}_f\in \mathbb{R}^{M \times M}$ as the parameters we want to learn, since they cannot be derived from the model-based formulation. On the other hand, the scalars $Q_v$, $Q_a$, and $P$ are model-based control parameters that are directly related to the response and convergence speed of the MPC. Therefore, we do not consider them as parameters to learn and can assign reasonable values directly. Additionally, it is important to note that $Q_v$, $Q_a$, and $P$ must be positive. {The positivity of these scalars is crucial for the feasibility of optimization problem~\eqref{eq:mpc_layer}. We will explain why these scalars must be positive in Theorem~\ref{thm1}.}

Optimization problem \eqref{eq:mpc_layer} is a quadratic program (QP). The work in \cite{amos2017optnet} demonstrates that a QP can be represented as a differentiable layer. In addition, a method is proposed to solve multiple QPs in batch form in \cite{amos2017optnet}. Therefore, the NN shown in Fig.~\ref{fig:model_pipline} can perform forward pass and backpropagation in batch form if the optimization problem \eqref{eq:mpc_layer} is feasible throughout the training process. To ensure feasibility, the resulting QP's Hessian matrix must remain symmetric positive definite when $\mathbf{A}_f$ and $\mathbf{Q}_f$ change during training \cite{boyd2004convex}.

Regarding the feasibility of our proposed differentiable MPC layer, we state the following theorem.

\begin{theorem}
\label{thm1}
If $Q_v,Q_a,P > 0$ and $\mathbf{Q}_f$ is symmetric positive definite, then the resulting QP from equations \eqref{eq:mpc_layer} and \eqref{eq:mpc_layer_tran} is feasible regardless of any changes to the dimension of the embedding $M$, the prediction horizon $N$, and $\mathbf{A}_f$.
\end{theorem}

\begin{proof}
Since the optimization problem \eqref{eq:mpc_layer} does not have constraints besides \eqref{eq:mpc_layer_tran}, the resulting QP is unconstrained.  Therefore, to prove the theorem, it is sufficient to demonstrate that the Hessian matrix of the resulting QP is symmetric positive definite \cite{boyd2004convex}. The resulting Hessian matrix from \eqref{eq:mpc_layer} is 
\begin{align}
\label{eq:H}
\mathbf{H} &= 2(\bar{\mathbf{Q}}_a + \bar{\mathbf{S}}^T\bar{\mathbf{Q}}\bar{\mathbf{S}}),\\
\text{where}~\bar{\mathbf{Q}}_a &=   \text{blkdiag}(Q_a,\dots,Q_a)\in \mathbb{R}^{N \times N},\notag\\
\bar{\mathbf{Q}} &=   \text{blkdiag}(\mathbf{Q},\dots,\mathbf{Q},P\mathbf{Q})\in \mathbb{R}^{N(M+1) \times N(M+1)},\notag\\
\mathbf{Q} &=   \text{blkdiag}(\mathbf{Q}_f,Q_v)\in \mathbb{R}^{(M+1) \times (M+1)},\notag\\
\bar{\mathbf{S}} &= 
\left[\begin{array}{cccc}
\mathbf{B}&\mathbf{0}&\dots&\mathbf{0}\\
\mathbf{A}\mathbf{B}&\mathbf{B}&\dots&\mathbf{0}\\
\vdots&\vdots&\ddots&\vdots\\
\mathbf{A}^{N-1}\mathbf{B}&\mathbf{A}^{N-2}\mathbf{B}&\dots&\mathbf{B}
\end{array}\right],\label{eq:rank}\\
\bar{\mathbf{S}} &\in \mathbb{R}^{N(M+1) \times N},\notag\\
\mathbf{A} &=  \left[\begin{array}{cc}
    \mathbf{I}^{M\times M}&\mathbf{A}_f\\
    \mathbf{0}^{1\times M}&1
\end{array}\right]\notag\in \mathbb{R}^{(M+1) \times (M+1)},\\
\mathbf{B} &=  \left[\begin{array}{c}
    \mathbf{0}^{M\times 1}\\
    \Delta t
\end{array}\right]\in \mathbb{R}^{M+1}.\notag
\end{align}

Because $Q_v,Q_a,P > 0$ and $\mathbf{Q}_f$ are symmetric positive definite, $\bar{\mathbf{Q}}$ and $\bar{\mathbf{Q}}_a$ are symmetric positive definite. Therefore, by equation \eqref{eq:H}, we can see that the Hessian matrix is symmetric positive semi-definite.

To prove that the Hessian matrix is symmetric positive definite, we can prove rank($\bar{\mathbf{S}} $)~$= N$. Because $\mathbf{B} \neq \mathbf{0}$, we can clearly see from equation \eqref{eq:rank} that rank($\bar{\mathbf{S}} $)~$= N$. Therefore, we have shown that the Hessian matrix of the resulting QP is symmetric positive definite, which completes the proof of the theorem.
\end{proof}

Theorem~\ref{thm1} demonstrates that if we ensure that $\mathbf{Q}_f$ is symmetric positive definite, the selection of $N$ and $M$, and uncertainty in $\mathbf{f}_n$ and $\mathbf{A}_f$ will not affect the QP's feasibility. As a result, we can freely tune the parameters and train the model stably. To ensure that $\mathbf{Q}_f$ remains symmetric positive definite during training, we utilize a Cholesky factorization
$$\mathbf{Q}_f = \mathbf{L}_f\mathbf{L}_f^T + \epsilon \mathbf{I}^{M\times M},$$
and directly learn $\mathbf{L}_f$, where $\mathbf{L}_f$ is a lower triangular matrix and $\epsilon$ is a very small scalar. We choose $\epsilon = 1 \times 10^{-4}$.

Although Theorem~\ref{thm1} shows that the selection of the prediction horizon $N$ and the embedding dimension $M$ does not affect the feasibility of the QP, it is still important to choose the appropriate values for these parameters. Very small $N$ and $M$ can negatively impact the network's ability to extract useful information and generalize well. Conversely, setting $N$ and $M$ too large can slow down the optimization problem's solution due to the increase in dimensionality when implementing the trained model as a controller, which can result in a low control frequency. Moreover, when $N$ is excessively large, the linear MPC layer may result in a large prediction error. Therefore, the choice of $N$ and $M$ should balance the expressiveness of the model and the computational efficiency of the controller. In our implementation, we choose $N=15$ and $M=20$ to strike this balance.

\begin{figure*}[ht]
\centering
\begin{overpic}[trim=0 0 0 0,clip, width=0.75\textwidth]{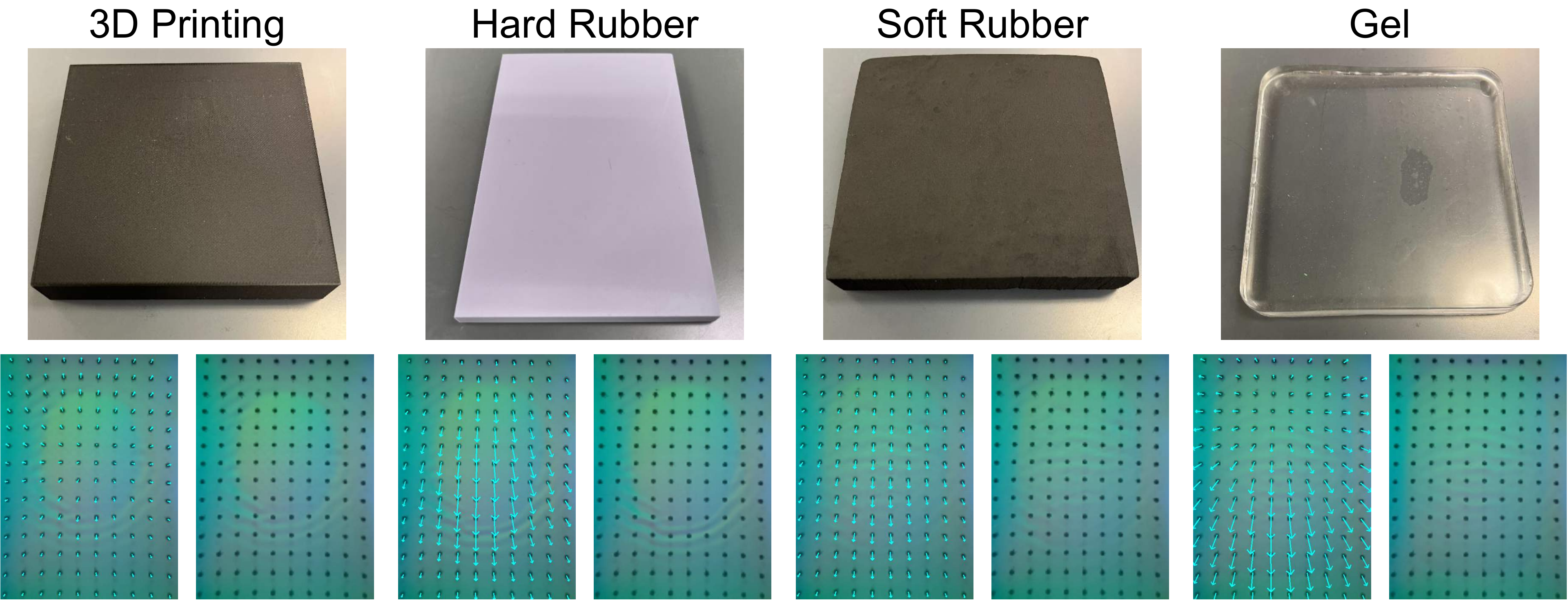}
\end{overpic}
\caption{For each type of material, bottom left is the raw tactile image overlapping with the marker tracking when grasping the corresponding block and bottom right is the raw tactile image without marker tracking. These tactile images are collected under the same grasping force. {Our goal here is to show that different materials have different physical properties, which leads to different features in the tactile image under the same grasping conditions. Therefore, we do not estimate the specific values of tangential force here.}}
\label{fig:data_collection_blocks}
\end{figure*}

\begin{figure*}[ht]
\begin{overpic}[trim=0 0 0 0,clip, width=0.35\textwidth]{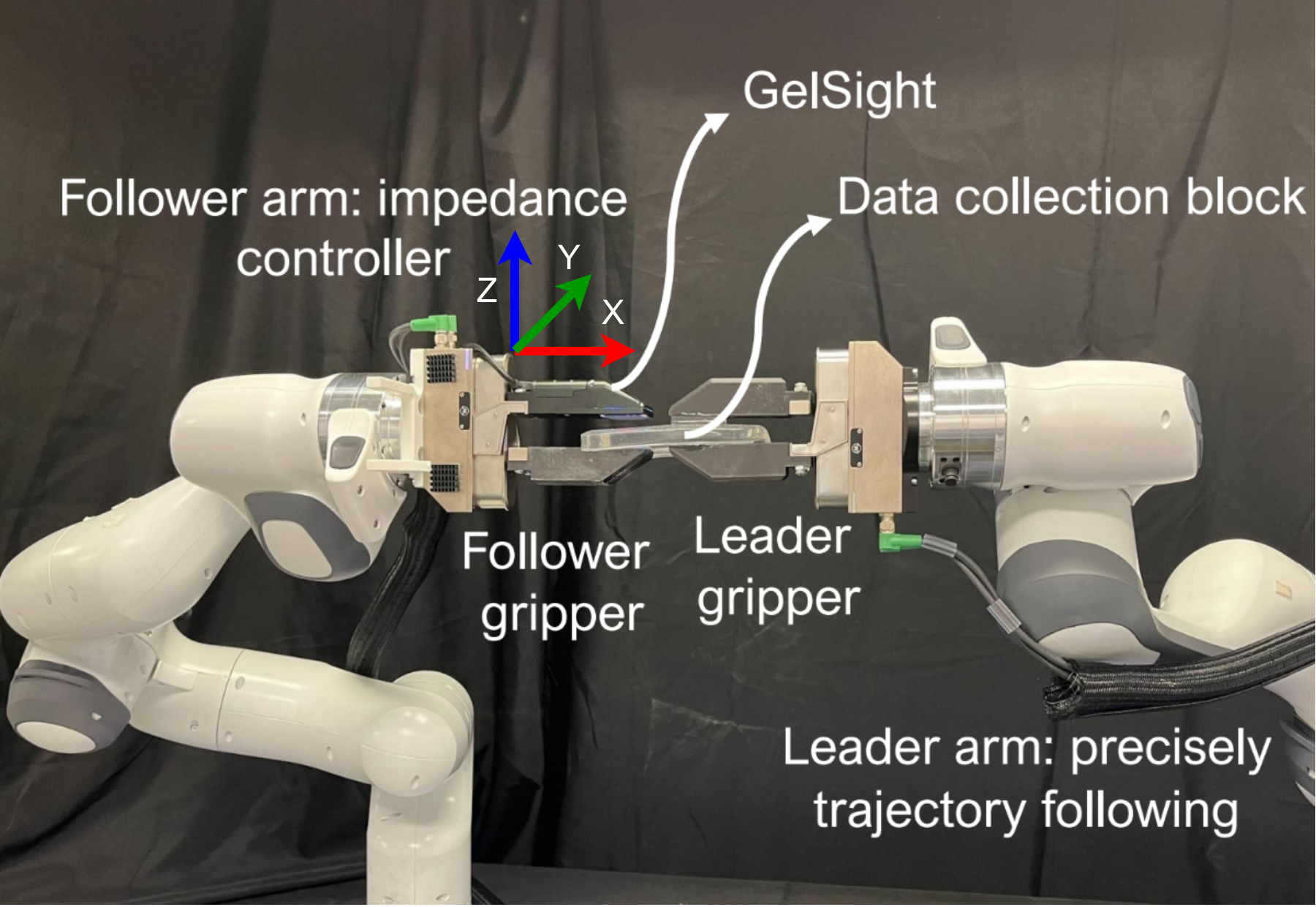}
\end{overpic}
\label{fig:data_collection_setup}
\begin{overpic}[trim=0 0 0 0,clip, width=0.64\textwidth]{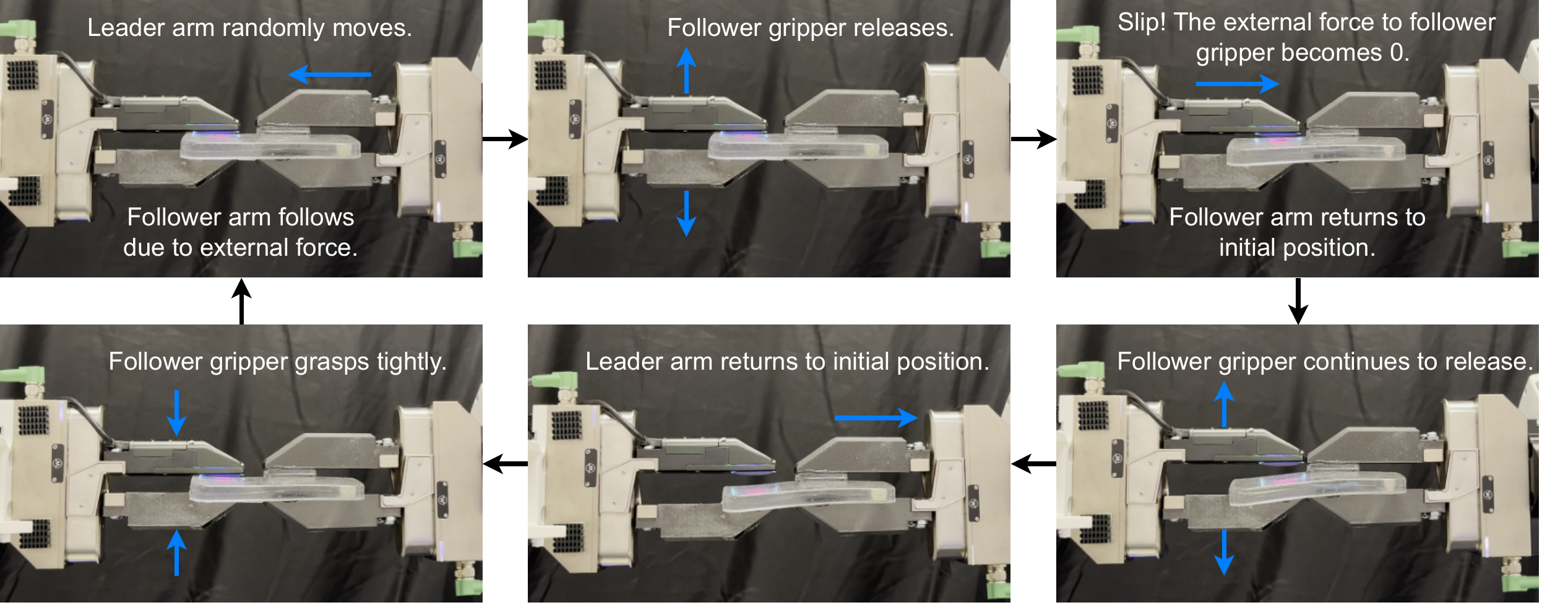}
\end{overpic}
\caption{Left: data collection setup. Right: automated data collection pipeline. {The $xyz$-coordinate of the end-effector's position increments is shown in the left figure.}}
\label{fig:data_collection_pipeline}
\end{figure*}

\begin{figure}[ht]
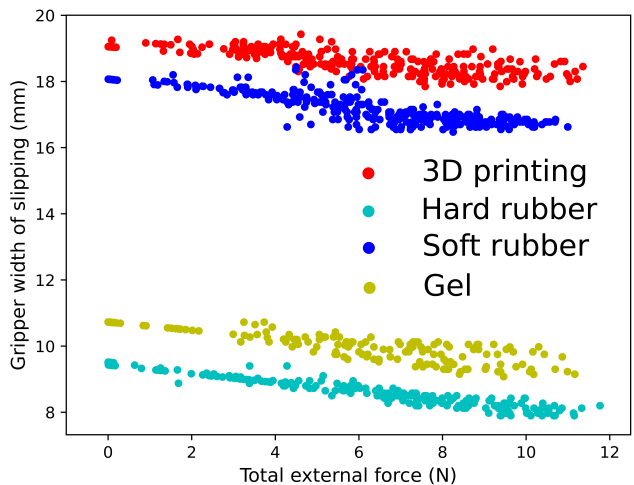

\centering
\begin{overpic}[trim=0 0 0 0,clip, width=0.48\textwidth]{images/data_visual.pdf}
\end{overpic}
\caption{Collected data visualization.}
\label{fig:data_distri}
\end{figure}

{ Denote the loss function as $\ell$ (see Section~\ref{subsec:loss_function}),  ``differentiable" in our formulation means,  gradients of  $\ell$  with respect to $\mathbf{A}_f$ and $\mathbf{L}_f$, denoted as $\nabla_{\mathbf{A}_f}^{\ell}$ and $\nabla_{\mathbf{L}_f}^{\ell}$,
can be computed with standard backpropagation. Additionally, by employing the batched QP solver introduced in \cite{amos2017optnet}, our proposed differentiable MPC layer can be integrated into a NN, enabling the training of the entire model in a standardized manner. }

\subsection{Data Collection}\label{subsec:data_collection}
{
We selected four different types of materials for data collection, as illustrated in Fig.~\ref{fig:data_collection_blocks}. It is important to mention that we chose only standardized blocks without surface texture and do not require specific sizes, as long as they fit the gripper. By doing so, the data collection process is simplified. we show in Section~\ref{sec:exp} that our proposed network model can learn from these standardized blocks and generalize to other daily objects with different textures, sizes, shapes, and physical properties.}

{
These blocks are used as grasped objects to capture raw tactile images. The tactile feedback of these blocks varies depending on their unique physical properties. As shown in Fig.~\ref{fig:data_collection_blocks}, we can see the differences in tactile feedback between these 4 blocks.}

The contact area of hard materials is smaller but more distinct, while the contact area of soft materials is larger but subtler. This is because of the difference of the objects' stiffness. When objects with different stiffness are grasped under the same applied force, the deformation of the grasped objects and the resulting deformation on the GelSight surface will be different, resulting in different raw tactile images.

The 3D printing block serves the role of a rigid object. We utilized a Markforged Onyx printer, which prints using micro carbon fiber filled nylon. This material is Onyx composite material. Its tensile modulus is 2.4 Gpa, and its flexural modulus is 3 Gpa \cite{3dpirnt}. When printed with a block form factor, it behaves rigid when grasped by our grippers. Compared with the hard block made by 3D printing, hard rubber is softer, so the contact area is larger. Additionally, the contact area of gel is the most subtle. The gel is made of XP-565 with shore A hardness of 27 \cite{gel1}, which is approximate 0.0009942 GPa of the Young's modulus \cite{gel2}. The stiffness of the softer rubber is between the gel and hard rubber. Finally, the difference in marker displacements can be attributed to the fact that materials with higher surface friction coefficients tend to induce a more pronounced displacement on the GelSight's surface. { It is worth noting that XP-565 is a commonly used material to make silicone gel for GelSight sensors.  The gel made of XP-565 has a nominal shore A hardness of 27 \cite{gel1}. Therefore, the stiffness of the gel in Fig.~\ref{fig:data_collection_blocks} is very close to the gel of the GelSight sensors made by XP-565. We believe that this point can serve as a benchmark for selecting data collection blocks: the chosen blocks usually possess a stiffness greater than that of the sensor gel. Otherwise, the deformation caused by the block on the sensor gel would be very subtle or even be invisible.} 

We believe that when other people select materials to construct a dataset, following our approach of choosing one hard material (such as rigid 3D printed objects or wood) and then selecting materials with progressively lower stiffness, up to a total of four materials, would be an effective process for dataset construction.

Our network input consists only of raw images. Since marker displacements are actually obtained by optical flow from raw images, the raw images themselves already contain this information. To collect data for training the NN shown in Fig.~\ref{fig:model_pipline}, we propose an automatic data collection pipeline utilizing a dual-arm setup. We use Franka Panda robotic arm. One of the arms serves as the {leader} arm, and we control it by making it follow specific trajectories. The other arm serves as the follower arm, and we implement an impedance controller for it. The WSG 50-110 grippers are attached to the end-effectors of both arms, with one finger on the follower arm mounting the GelSight tactile sensor.

A complete data collection trial is illustrated in Fig.~\ref{fig:data_collection_pipeline}. We perform multiple iterations of this trial to collect the necessary data. At the start of each trial, both grippers securely grasp the object. We randomly select position increments $\Delta x$ and $\Delta y$ within the range of $|\Delta x| < 35~\text{mm}$ and $|\Delta y| < 21~\text{mm}$, and apply these increments to the end-effector of the leader arm. The leader arm then moves, and the follower arm follows its motion due to the impedance controller. As a result, an external force is horizontally applied to the grasped object, with the magnitude and direction determined by the deviation $\Delta x$ and $\Delta y$ of the follower manipulator's end-effector from the equilibrium point of its impedance controller.

Next, the follower gripper gradually relaxes with a velocity of 4.5~mm/s, and we record the width of the follower gripper $p_n$ and the raw image of the tactile feedback at 60~Hz.  When the follower gripper's grasping force applied to the object cannot provide a maximum static friction force greater than the external force from the impedance controller, slipping occurs between the follower gripper and the object. This results in the follower arm moving back to the equilibrium point of the impedance controller.  We record the width of the follower gripper when slipping occurs as the label for all raw images recorded during the entire trial, denoted as $p_{\text{slip}}$. 

Ideally, we hope that the leader gripper maintains an absolutely tight grasping throughout the process and there is no slipping between the leader gripper and the grasped object. To achieve that, we maximize the grasping force of the leader gripper on the data collection block and enhance the friction between the leader gripper’s fingers and the block by adding a gel layer on the finger. Naturally, we could not guarantee that there would be no movement between the object and the leader gripper during extended periods of data collection. However, our setup ensures that any movement, if it occurred, is minimal and does not affect data collection. Therefore, when the leader gripper stays static, the grasped object stays static and we set the velocity of the end-effector of the follower arm to be greater than 5~mm/s as the condition for indicating slipping. Once slipping occurs for 0.2~s, we stop recording the raw images. {The reasons of selecting 0.2~s are as follows. For each material, as shown in Fig.~\ref{fig:data_distri}, when the external force varies from 0 to 12N, the fluctuation range of the $p_{\text{slip}}$ collected for each material is approximately within 2~mm. Using 2~mm as a reference, to ensure the data collected in a reasonable range (for example, to avoid collecting too many images without contact), we chose a distance of 0.9~mm  (0.2~s $\times$ 4.5~mm/s) as the gripper width range of one data collection trial. It is worth noting that the choice of 0.2~s here is approximate. The goal is to prevent the collected data from becoming overly biased.}

For the gripper velocity $v_n$ in the dataset, we randomly select a value within the range of $(-1~\text{mm/s},1~\text{mm/s})$ for each data point. As shown in Fig.~\ref{fig:model_pipline}, $p_n, v_n$, and the raw tactile image are the inputs of our model.

Finally, the follower gripper continues to relax until the object is completely released. After that, the leader arm returns to its initial position, and the follower gripper tightly re-grasps the object. The next trial then begins. 

{The Franka Panda robot features force-torque feedback of the end-effector, eliminating the need for additional sensors.} We simultaneously record force feedback during the data collection process. Although force feedback is not useful for model training, we can visualize its relationship with $p_{\text{slip}}$ to verify the reasonableness of the collected data, as shown in Fig.~\ref{fig:data_distri}. {Fig.~\ref{fig:data_distri} shows that the gripper width and the external force applied to the block are approximately negatively correlated. This indicates an approximate linear positive correlation between the pressure and the total external force applied to each block, aligned with the Coulomb friction model. According to this model, the maximum static friction force $F_{\text{max}} = \mu F_{\text{normal}}$, where $F_{\text{normal}}$ is the normal force exerted on the contact surface and $\mu$ is the coefficient of static friction. Additionally, we know that the gripper width and  $F_{\text{normal}}$ are negatively correlated (a smaller gripper width means tighter grasping and a larger $F_{\text{normal}}$). Therefore, the depiction in Fig.~\ref{fig:data_distri} corroborates this equation.} This demonstrates the effectiveness of our automated data collection method.

We also record some raw tactile images when not applying external force. For these images and some cases where the external force is too small, there is no slipping occurring. We run a linear regression on the remaining data points of this material to estimate $p_{\text{slip}}$ for these data points, as shown in Fig.~\ref{fig:data_distri}. We also add some tactile images without any contact to the dataset. For images without contact, we set $p_n = 30~\text{mm}$, $p_{\text{slip}} = 28.5~\text{mm}+\text{RAND}(-0.5~\text{mm},0.5~\text{mm})$, and $v_n = (p_{\text{slip}} - p_n)/3$, where $\text{RAND}(-0.5~\text{mm},0.5~\text{mm})$ generates a random distance between -0.5~mm and 0.5~mm. 3 refers to 3 seconds. This parameter controls the desired velocity of gripper width when there is no contact.
 A lower value might result in an excessively high desired velocity, which may not be feasible for low-level tracking. If the value is set too high, the gripper contracts too slowly when there is no contact, resulting in a significant decrease in overall system efficiency. Therefore, the choice of 3 seconds is also based on experiments tailored to our hardware specifications. During training, these images without contact allow the gripper to continue closing until it makes contact with the object.

{We mount only one GelSight sensor on the follower gripper, and the other finger is with the same shape and no sensor mounted. Similarly, the two fingers on the leader gripper are all dummy fingers without mounted sensors since the leader arm is used to provide external forces during data collection. The reason for using only one GelSight on one side is because a single GelSight sensor suffices to furnish LeTac-MPC with adequate information for effective grasping. For the parallel gripper, the GelSight sensors mounted on different fingers can offer detailed information on the shape and texture of an grasped object's opposing sides, but often provide overlapping information about the information of contact force, and states and properties of the grasped object. GelSight only can perceive a very small local contact area of the object, and the idea of LeTac-MPC is using this small local area (raw tactile image) to estimate the grasping state and object property implicitly and then generate the corresponding actions. Therefore, we only mount one GelSight sensor, both for data collection and controller deployment. In addition, adding a redundant GelSight sensor would increase the computational load for real-time processing, subsequently reducing the control frequency. In addition, we would like to highlight the following papers \cite{she2021cable,sunil2023visuotactile}. Both of them have only one finger equipped with GelSight sensor but can accomplish complex manipulation tasks.}

\subsection{Loss Function}\label{subsec:loss_function}

We propose the following loss function for our model:
$$
\ell = D(\mathbf{p}_n^*,\mathbf{p}_{\text{slip}}) + \tilde{P}D(p_{n+N}^*,p_{\text{slip}}),
$$  where  $\mathbf{p}^*_n = [
p^*_{n+1},p^*_{n+2},\ldots{},p^*_{n+N}
]^T\in\mathbb{R}^{N}$, $\mathbf{p}_{\text{slip}} = [
p_{\text{slip}},p_{\text{slip}},\ldots{},p_{\text{slip}}
]^T\in\mathbb{R}^{N}$, and $D(\cdot,\cdot)$ is a distance measure. {We use L2 distance in our implementation.} The loss function measures the difference between the output predicted trajectories and the target position $p_{\text{slip}}$ over the sequence of output positions instead of comparing a single output position to $p_{\text{slip}}$. This sequence-based approach promotes the convergence of the controller. Additionally, we use coefficient $\tilde{P}$ to enhance the convergence of the trained controller by increasing the terminal cost. In our implementation, we set $\tilde{P} = 3$.

Recall that in data collection, $p_{\text{slip}}$ represents the maximum gripper width required for the gripper to hold the data collection blocks without slipping when external forces are present. Selecting $p_{\text{slip}}$ as the target position means aiming for training a model that can output the maximum gripper width that can maintain stability of the grasped object in the hand, which corresponds to applying minimal gripping force to the grasped object.

\begin{figure*}[ht]
\centering
\begin{overpic}[trim=0 0 0 0,clip, width=0.8\textwidth]{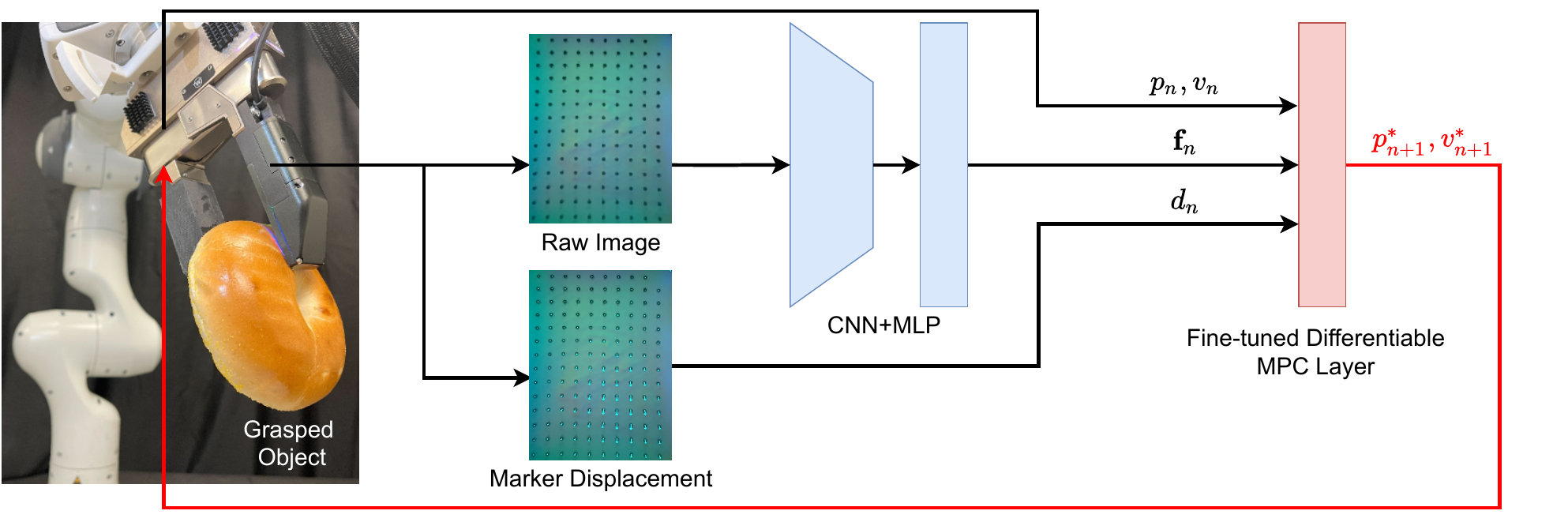}
\end{overpic}
\caption{Illustration of learned controller implementation. We implement the trained model as a real-time MPC controller. The proposed MPC layer combines model-based and data-driven formulations, allowing for the fine-tuning of the model-based part to achieve better control performance with real-time tactile feedback.}
\label{fig:controller_piplone}
\end{figure*}

\subsection{Controller Implementation} \label{subsec:imp}
{The model is trained using discrete data points, while the actual tactile feedback is continuous and real-time. This difference in data formats poses significant challenges in deploying a trained model as a real-time controller with optimal performance.} In our method, the proposed MPC layer combines both model-based and data-driven formulations, providing a way to fine-tune the model-based part for better control performance with real-time tactile feedback. To implement the trained model as a controller, we need to make three adjustments to the MPC layer, as shown in Fig.~\ref{fig:controller_piplone}. Firstly, we increase the weight coefficients $Q_v$ and $Q_a$ by multiplying the scalars $K_v$ and $K_a$, respectively, to increase the weights coefficients to speed up the controller's convergence. Secondly, we incorporate an additional model-based feedback signal, the marker displacement, to enhance the response speed to the tangential contact force. Finally, we add gripper motion saturation constraints in the optimization problem to ensure that control input and output are dynamically feasible. The implemented MPC controller is determined by solving the optimization problem

\begin{align}
\label{eq:mpc_controller_imp}
        \mathbf{a}_n^*=&\arg \min _{\mathbf{a}_n}          P\mathbf{f}_{n+N}^T \mathbf{Q}_f \mathbf{f}_{n+N}\notag\\
        &+
        PK_vQ_v (v_{n+N}-K_dd_{n+N})^2\notag\\
    &+\sum_{k=n}^{n+N-1}
        \mathbf{f}_{k}^T \mathbf{Q}_f \mathbf{f}_{k}+
        K_vQ_v (v_{k}-K_dd_k)^2+K_aQ_a a_{k}^2,\\
        &\text{subject to}~\eqref{eq:mpc_layer_tran}~\text{and}~\notag\\
        &\left[\begin{array}{c}
            p_{\text{min}} \\
           v_{\text{min}} \\
            a_{\text{min}}
            \end{array}\right]\leq
         \left[\begin{array}{c}
            p_n \\
            v_n \\
            a_n
            \end{array}\right]\leq
        \left[\begin{array}{c}
            p_{\text{max}} \\
           v_{\text{max}} \\
            a_{\text{max}}
        \end{array}\right]\label{eq:saturation},
\end{align}
where the scalar factor $K_d$ is used to regularize $d_n$, which represents the length of the sum of marker displacements after applying the compensation method introduced in \cite{ma2019dense}.  The scalars $K_v$ and $K_a$ are used to increase the weight coefficients for $v$ and $a$, respectively.  {In our implementation, by setting $K_v = 100$~and~$K_a = 2$, we are able to achieve good convergence ability and response speed for the deployed controller. Inequality \eqref{eq:saturation}} is the saturation constraint, where right subscripts min and max represent the lower bound and upper bound, respectively.

The reason we add constraints during the implementation of the controller instead of incorporating them during the learning phase is that an unconstrained MPC layer can better capture the features of different objects to build more applicable dynamic models. When utilizing the learned model as a controller, the application of constraints enables the entire model to better adapt to real-time tactile feedback and generate dynamically feasible motion.

\section{Baseline Methods}

In this section, we introduce three baseline methods we use in the experiments, namely proportional derivative (PD) control, MPC, and open-loop grasping.  

\subsection{Proportional Derivative Control}

Firstly, we define several control signals that can be extracted from tactile images. We denote the contact area of the grasped object with the tactile sensor as $c$. The contact area $c$ can be obtained by thresholding the depth image or the difference image, as illustrated in Figs.~\ref{fig:thresholding}~and~\ref{fig:diff_thresholding}. Additionally, we estimate the tangential contact force by tracking the marker displacements. We represent the length of the sum of marker displacements as $d$.  

It should be noted that $d$ and $c$ here are tactile signals with clear meanings, extracted by other modules, and low-dimensional. This is distinct from \(\mathbf{f}_{n}\) used in LeTac-MPC. \(\mathbf{f}_{n}\) is an embedding generated by the encoder, which lacks a specific explicit meaning. During end-to-end training, we anticipate that our model would learn the physical representations of grasped objects from tactile images, which is \(\mathbf{f}_{n}\).

We design the controller according to two objectives. Firstly, we aim to maintain a constant value $c_{\text{ref}}$ for the contact area $c$ reactively while grasping objects. Secondly, we aim to ensure a tighter grasping in the presence of external tangential contact force to prevent the object from slipping. To achieve these goals, we define the following PD control law:

$$v_n = K_P(c_n-c_{\text{ref}}-Q_dd_n)+K_D\dot{c_n}.$$

Where $K_P, K_D$ are gains for this controller and $Q_d$ is a scalar factor to regularize $d$. The parameters we use for PD are shown in Table~\ref{tab:pd}.

\subsection{Model Predictive Control}\label{sec:model_based}

\begin{table}[t]
\centering
\begin{tabular}{llllll}
\hline
   $c_{\text{ref}}$  & $Q_d$ & $K_P$   &$K_D$ & $\Delta t$&freq.   \\ \hline
     900    & 2    & $\frac{1}{4\times 10^4 \Delta t}$ & $\frac{1}{3.6\times 10^5 \Delta t}$ & $\frac{1}{60}$ & 60~Hz   \\ \hline
\end{tabular}
\caption{{Parameters of the PD controller.} }
\label{tab:pd}
\end{table}

\begin{table}[t]
\centering
\begin{tabular}{lllllllll}
\hline
   $c_{\text{ref}}$  & $Q_d$ & $P$   &$Q_c$ & N& $K_c$&$Q_v,Q_a$  &$\Delta t$&freq.   \\ \hline
     900    & 2    & 5 & 36   &10&$36000$ &1 & $\frac{1}{60}$ & 60~Hz   \\ \hline
\end{tabular}
\caption{{Parameters of the MPC controller.} }
\label{tab:mpc}
\end{table}

\begin{table}[t]
\centering
\begin{tabular}{lll}
\hline
   $p_{\text{min}}/p_{\text{max}}$ & $v_{\text{min}}/v_{\text{max}}$ & $a_{\text{min}}/a_{\text{max}}$   \\ \hline
     $0/70~\text{mm}$    & $-15/15~\text{mm/s}$    & $-100/100~\text{mm/}\text{s}^2$    \\ \hline
\end{tabular}
\caption{{Constraints of LeTac-MPC and MPC.} }
\label{tab:values_cons}
\end{table}

\begin{figure}[t!]
\centering
\begin{overpic}[trim=100 10 90 10,clip, width=0.48\textwidth]{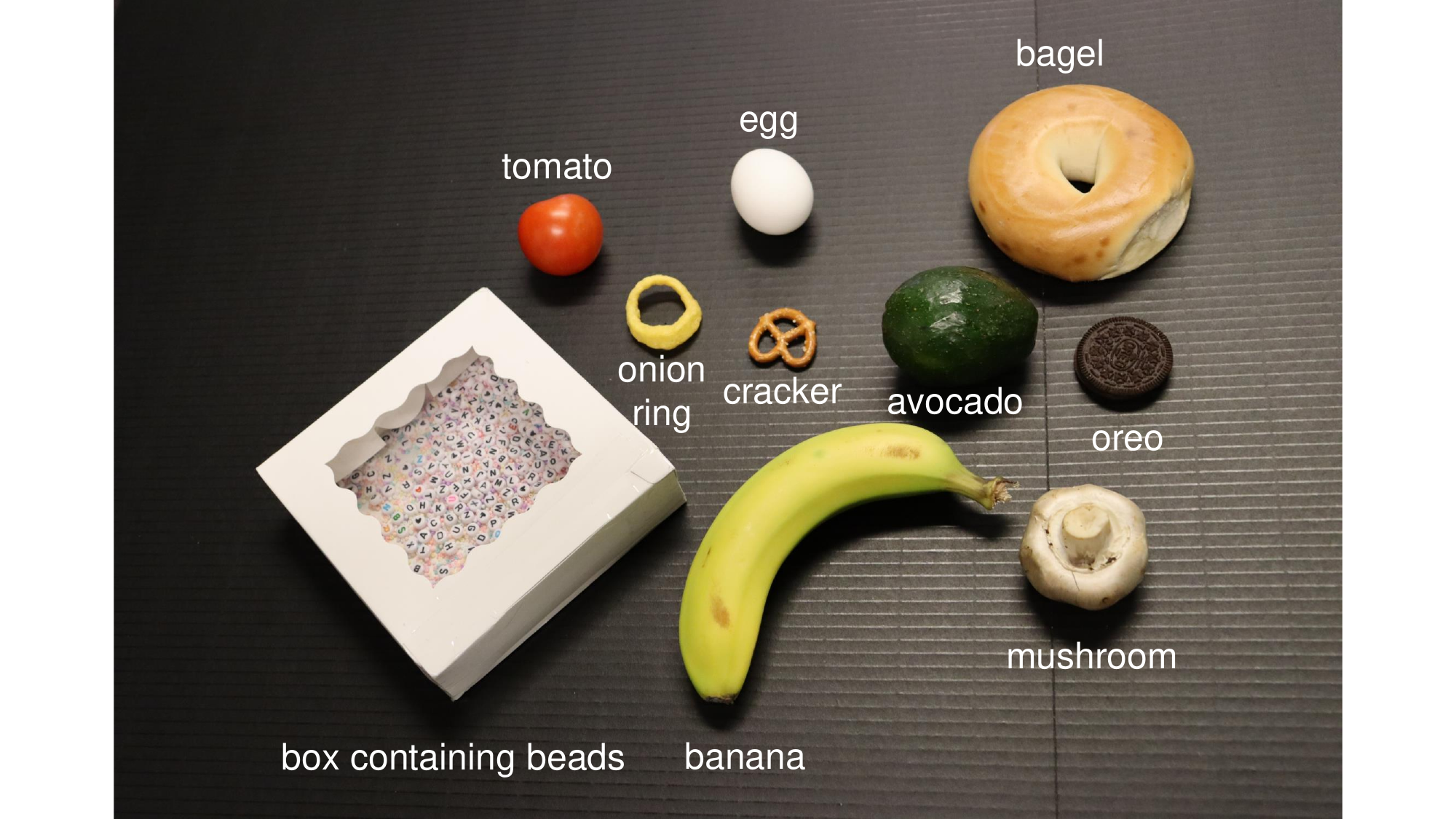}
\end{overpic}
\caption{Daily objects for grasping.}
\label{fig:daily_object}
\end{figure}

\begin{figure*}[ht!]
\centering
\begin{overpic}[trim=0 0 0 0,clip, width=0.85\textwidth]{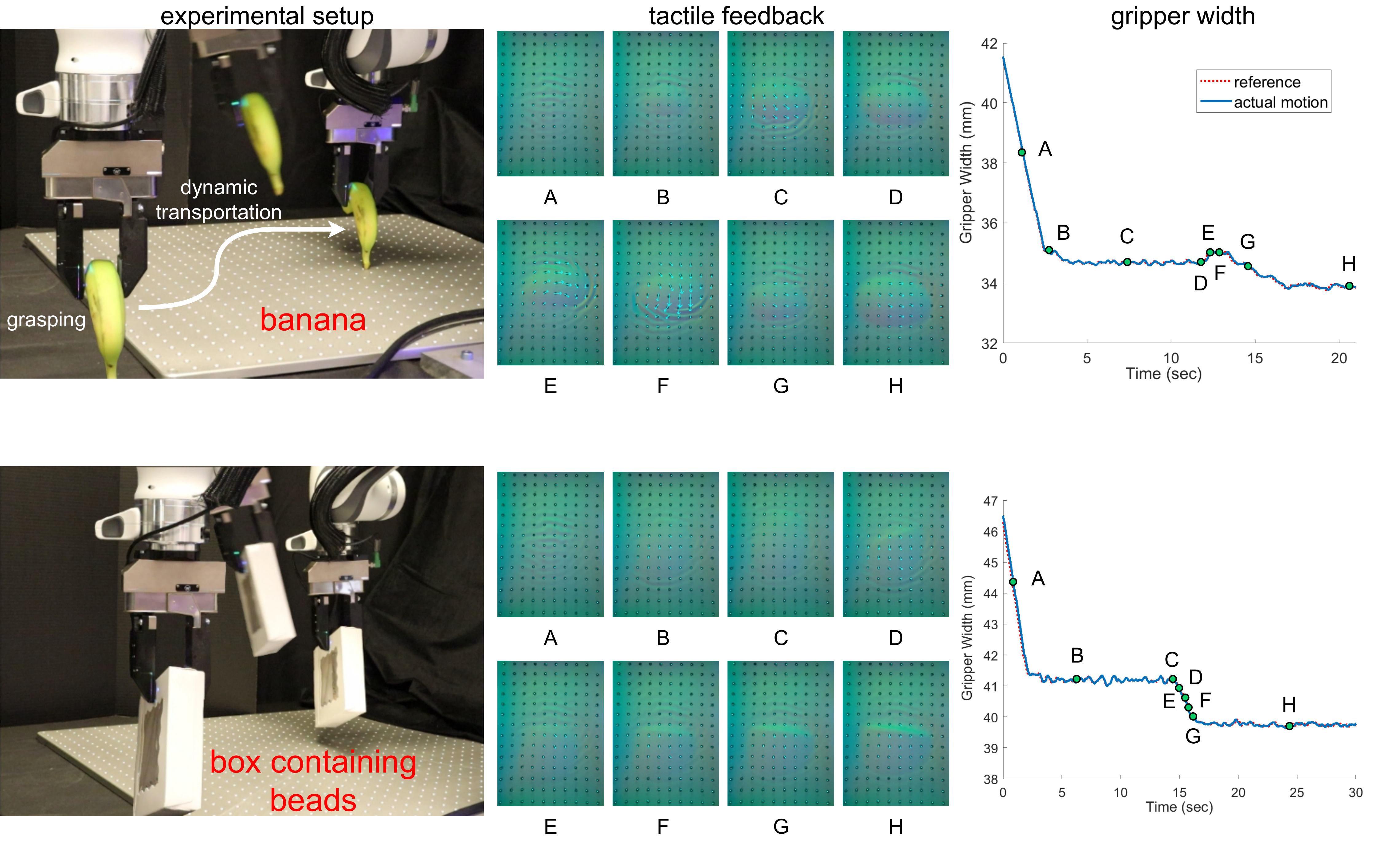}
\end{overpic}
\caption{Illustration of daily objects grasping and transportation. Left: experiment procedures. Middle: tactile feedback sequence. Right: gripper width curve. The dynamic trajectories we use for transportation are same for every object.}
\label{fig:banana}
\end{figure*}

\begin{figure*}[ht]
\centering
\begin{overpic}[trim=0 0 0 0,clip, width=0.75\textwidth]{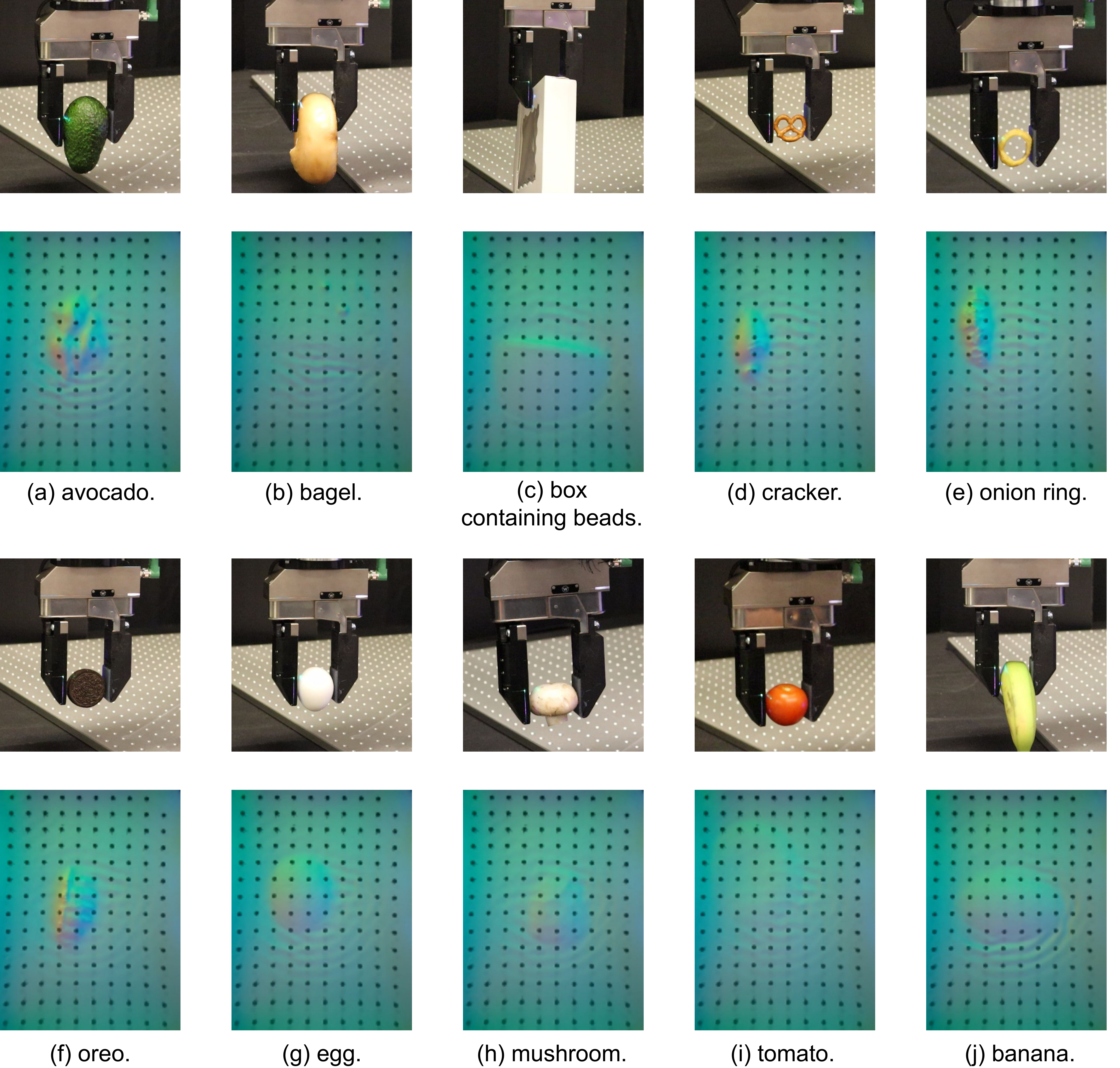}
\end{overpic}
\caption{Tactile images of daily objects grasping using LeTac-MPC.}
\label{fig:daily_grasping}
\end{figure*}

\begin{figure*}[ht]
\centering
\begin{overpic}[trim=0 0 0 0,clip, width=0.93\textwidth]{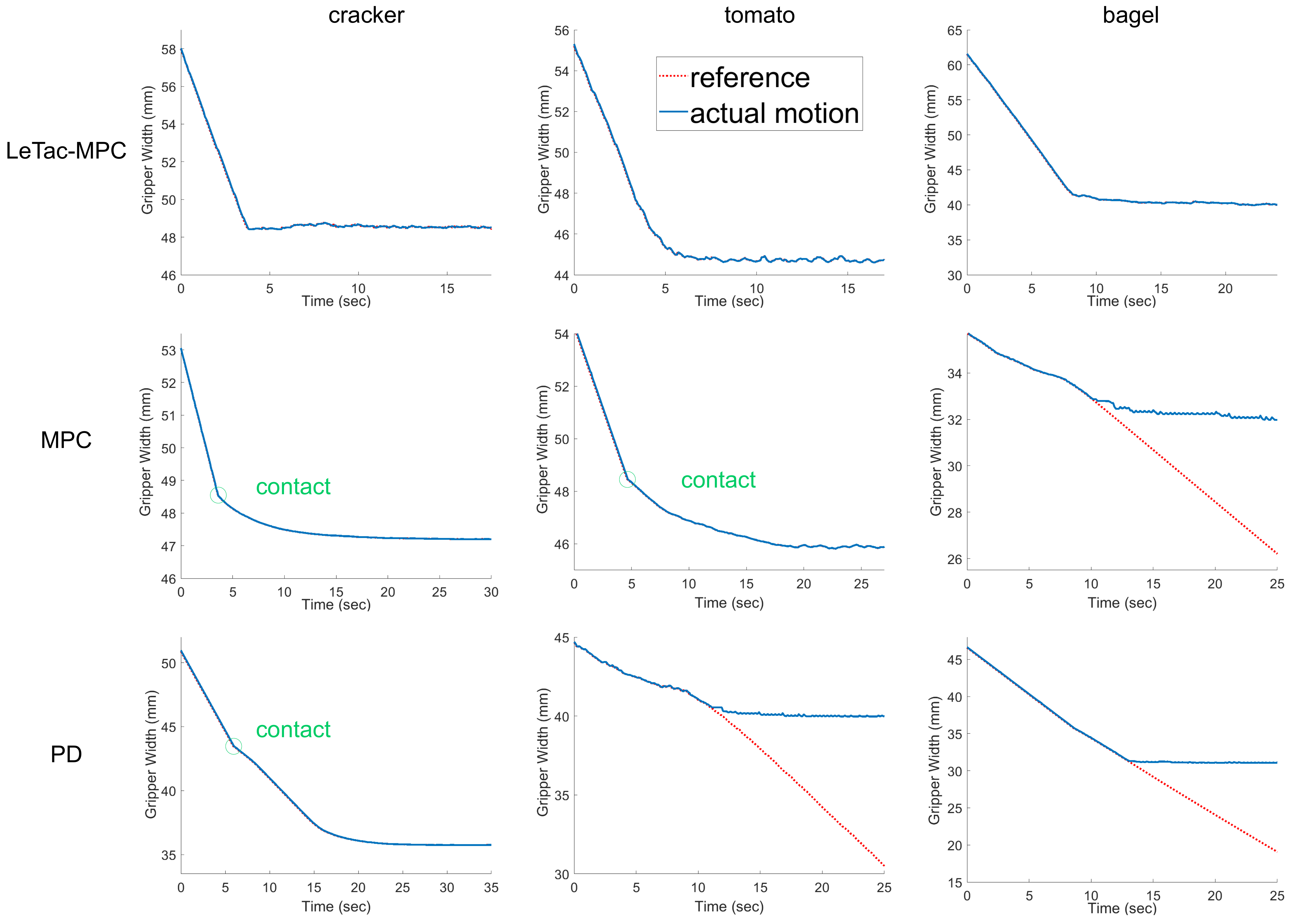}
\end{overpic}
\caption{Gripper width curves of different objects and different methods. }
\label{fig:grasping_curve}
\end{figure*}

\begin{table*}[ht]
\centering
\begin{tabular}{llllll}
\hline
          & avocado & bagel & box  & cracker & onion ring \\\hline
LeTac-MPC & \textbf{5.53}$\pm 1.48$~N& \textbf{5.32}$\pm 1.63$~N& \textbf{5.36}$\pm 1.57$~N& \textbf{4.44}$\pm 1.30$~N& \textbf{4.89}$\pm 1.34$~N \\
MPC       & 5.79$\pm 1.33$~N & \red{reach force limit} & 7.04$\pm 1.82$~N & 5.47$\pm 1.29$~N & 6.46$\pm 1.55$~N  \\
PD        & 5.96$\pm 1.40$~N & \red{reach force limit} & 7.85$\pm 2.00$~N & 6.88$\pm 1.50$~N & 5.83$\pm 1.37$~N  \\
open-loop & 10.22$\pm 2.02$~N & 7.11$\pm 1.85$~N      & \red{drop}        & 9.92$\pm 1.88$~N & \red{broken}  
\\\hline
 & oreo & egg & mushroom & tomato            & banana\\\hline
 LeTac-MPC & \textbf{4.67}$\pm 1.43$~N & 5.43$\pm 1.42$~N& \textbf{4.52}$\pm 1.51$~N& \textbf{5.02}$\pm 1.70$~N& \textbf{5.89}$\pm 1.68$~N\\
MPC        & 5.82$\pm 1.37$~N & 5.28$\pm 1.34$~N          & 5.31$\pm 1.40$~N & 7.42$\pm 1.66$~N       & 9.23$\pm 1.80$~N  \\
PD        & 6.69$\pm 1.46$~N & \textbf{4.53}$\pm 1.22$~N & 5.78$\pm 1.44$~N & \red{reach force limit} & 8.17$\pm 1.85$~N  \\
open-loop  & 7.78$\pm 1.89$~N    & 13.65$\pm 2.86$~N   & 7.02$\pm 1.75$~N        & 7.55$\pm 1.77$~N  & 7.35$\pm 1.83$~N  
\\\hline
\end{tabular}
\caption{Results of daily objects grasping and transportation experiment. If an experiment fails, we indicate the corresponding table position with red font to highlight the failed result. If the experiment succeeds, {we report the average value and standard deviation of the force during the experimental process after the grasping stabilizes.} The bolded font represents the minimum grasping force for the same object. {Here, the term ``egg" refers to an egg with its shell, which is actually rigid. In the paper, we use ``peeled boiled egg" to denote an egg without its shell.}}
\label{tab:force}
\end{table*}

\begin{figure*}[ht]
\centering
\begin{overpic}[trim=0 0 0 0,clip, width=0.8\textwidth]{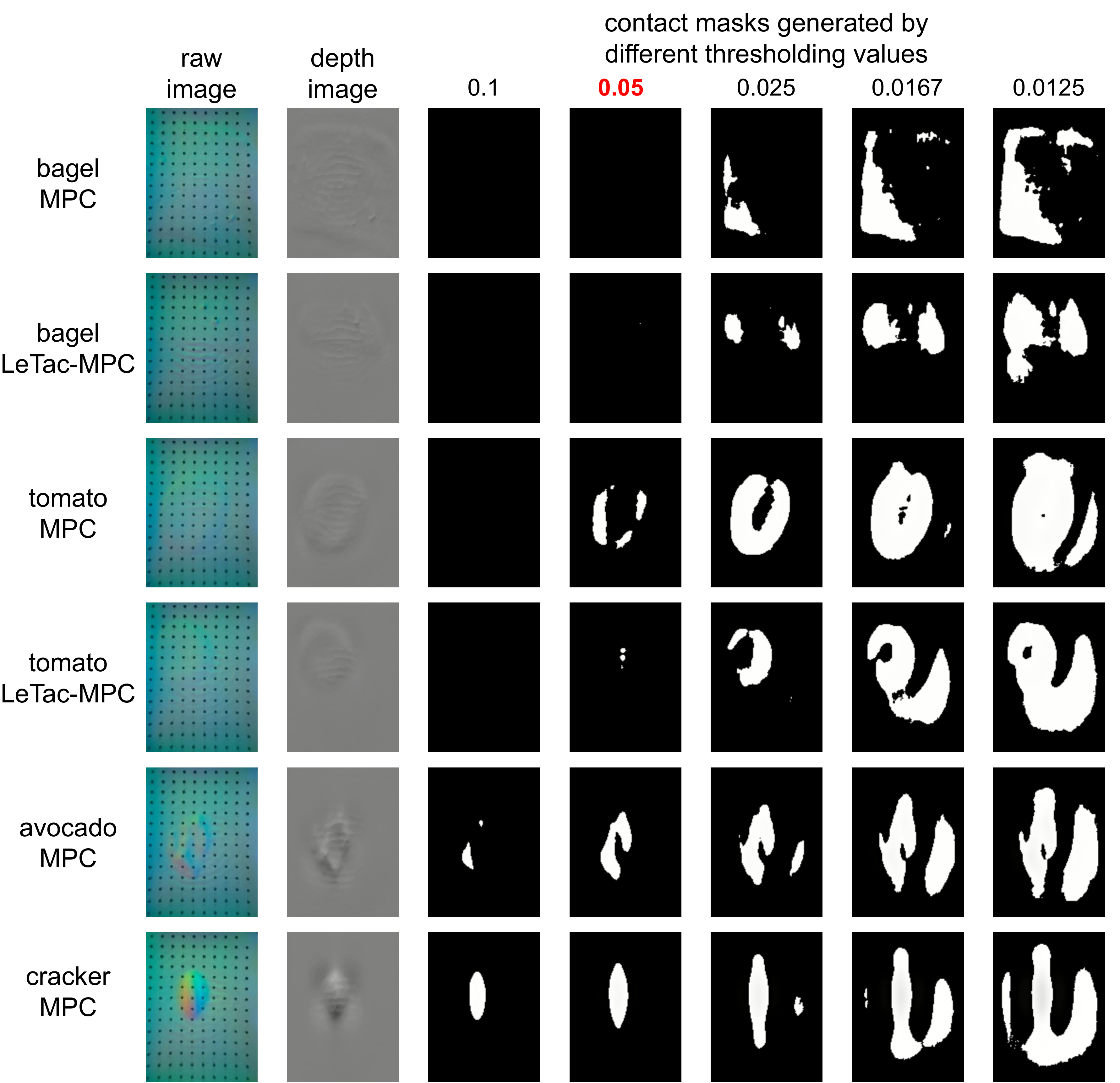}
\end{overpic}
\caption{Thresholding the depth image to extract the contact area. We show a sequence of results by thresholding with different values. 0.05 is the threshold we use for PD and MPC experiments. }
\label{fig:thresholding}
\end{figure*}

\begin{figure*}[ht]
\centering
\begin{overpic}[trim=0 0 0 0,clip, width=0.8\textwidth]{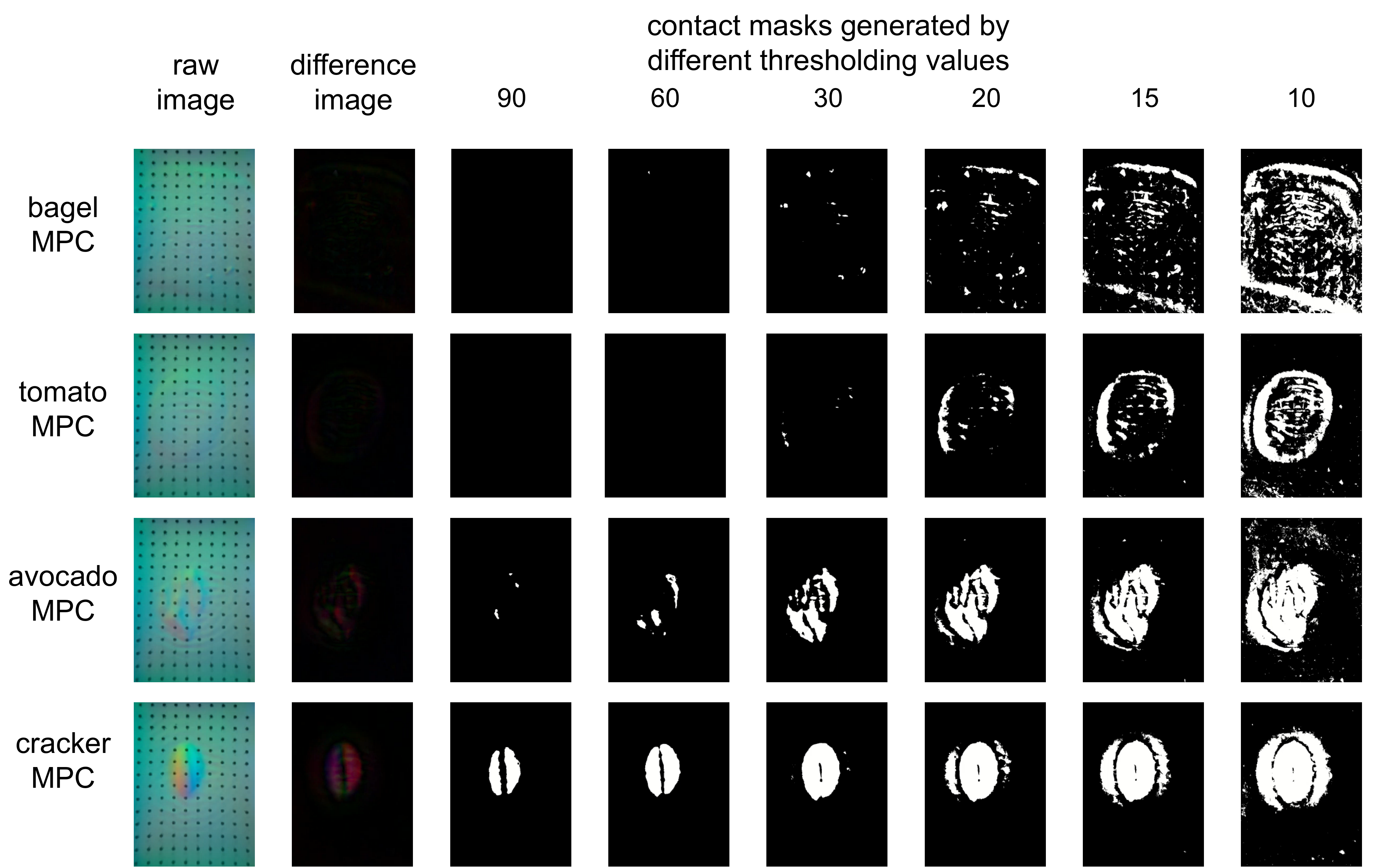}
\end{overpic}
\caption{Thresholding the difference image to extract the contact area. To get the difference image, we need to first interpolate the raw image and remove the markers.}
\label{fig:diff_thresholding}
\end{figure*}

\begin{figure*}[ht]
\centering
\begin{overpic}[trim=0 0 0 0,clip, width=0.8\textwidth]{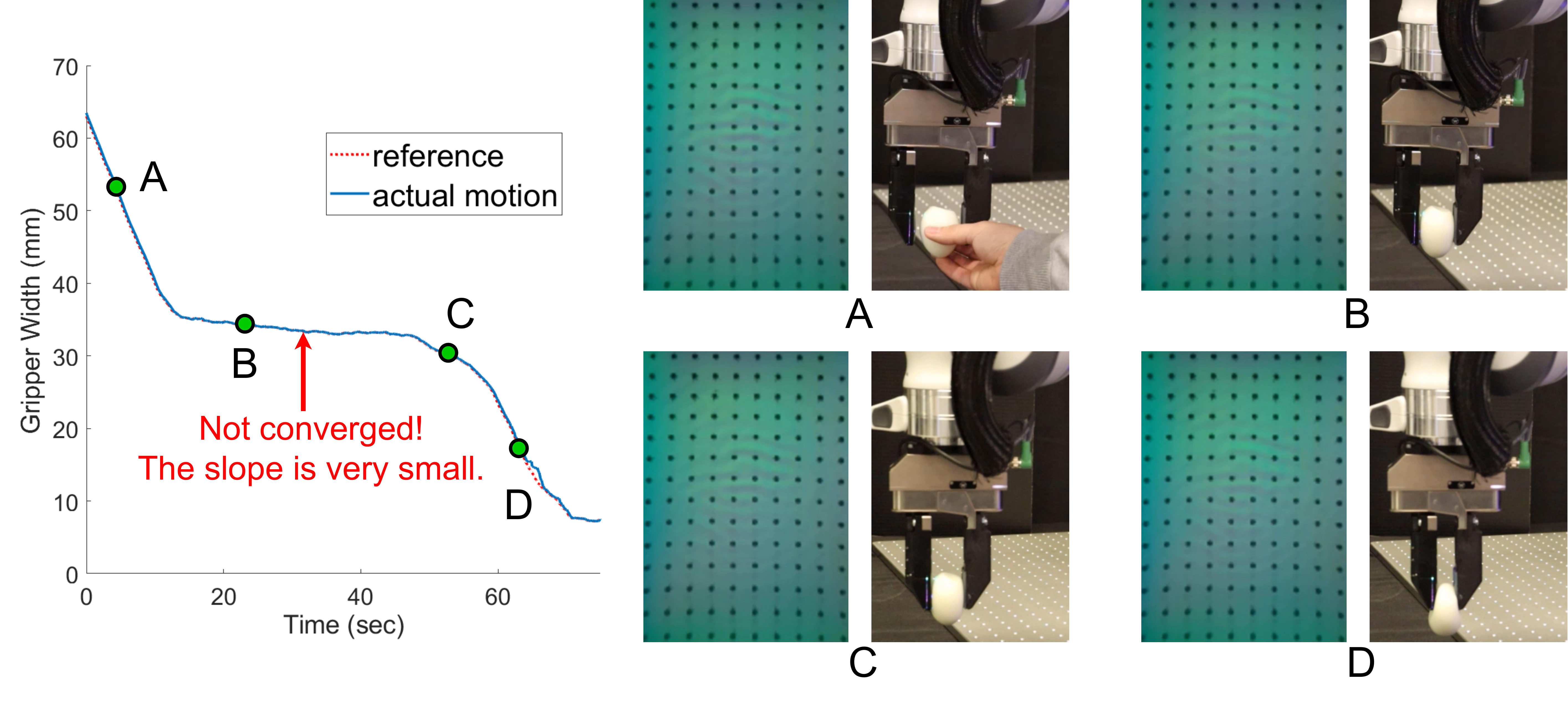}
\end{overpic}
\caption{Gripper width curves and tactile images of grasping a peeled boiled egg using LeTac-MPC. }
\label{fig:egg_fail}
\end{figure*}

\begin{figure*}[ht]
\centering
\begin{overpic}[trim=0 0 0 0,clip, width=0.9\textwidth]{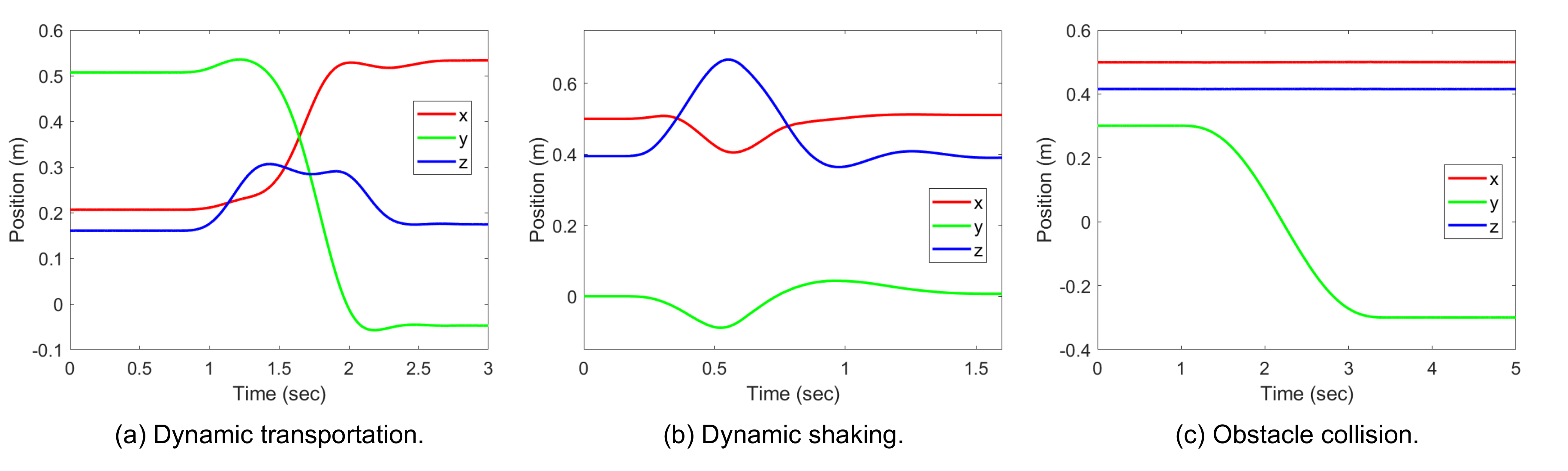}
\end{overpic}
\caption{{Position trajectories for the dynamic transportation (Section~\ref{sec:grasping}), dynamic shaking (Section~\ref{sec:shaking}), and obstacle collision (Section~\ref{sec:obs}) experiments. The maximum velocity of the trajectory for dynamic transportation is 1.40 m/s, with a maximum acceleration of 3.82 m/$\text{s}^2$. For dynamic shaking, the trajectory has a maximum velocity of 1.38 m/s and a maximum acceleration of 8.15 m/$\text{s}^2$. Lastly, the trajectory for obstacle collision has a maximum velocity of 0.47 m/s and a maximum acceleration of 0.65 m/$\text{s}^2$. } {This figure serves as a further supplement to the experiments, highlighting the dynamic nature and challenges of the trajectories involved.}}
\label{fig:traj}
\end{figure*}

\begin{figure*}[ht]
\centering
\begin{overpic}[trim=0 0 0 0,clip, width=0.9\textwidth]{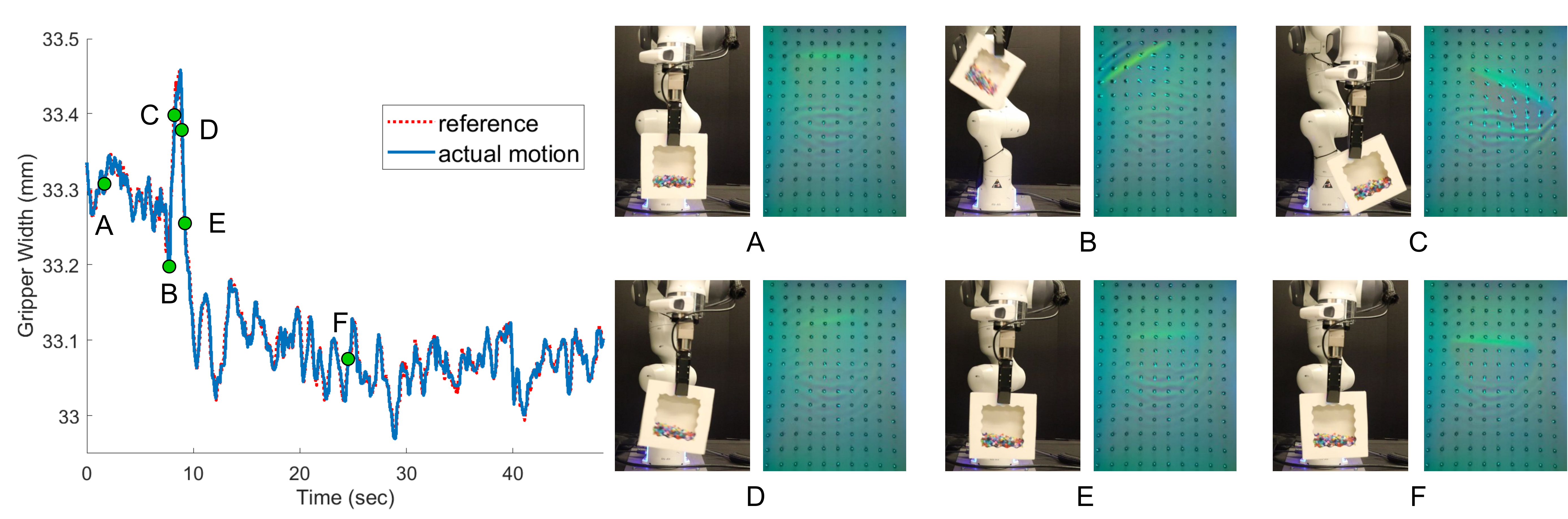}
\end{overpic}
\caption{Gripper width curve and tactile images of dynamic shaking a box containing beads with LeTac-MPC. }
\label{fig:nn_shaking}
\end{figure*}

\begin{figure*}[ht]
\centering
\begin{overpic}[trim=25 0 20 0,clip, width=0.99\textwidth]{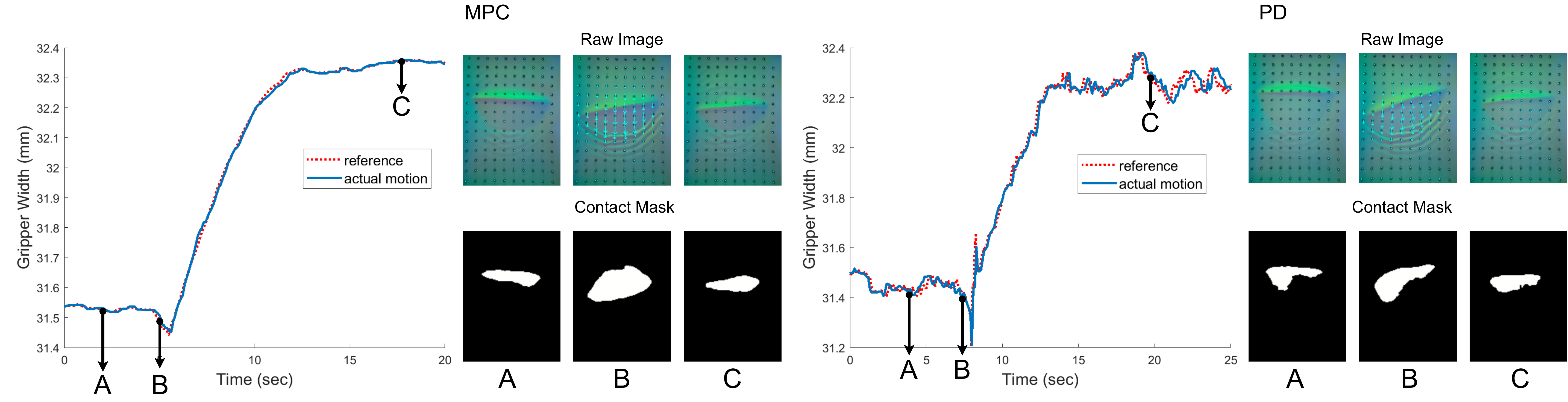}
\end{overpic}
\caption{Gripper width curves of dynamic shaking experiments with MPC and PD control. {We know that the gripper width and grasping force are negatively correlated (smaller gripper width means tighter grasping and larger grasping force). We can see from this figure that at the beginning of the experiments, the contact areas are too small compared to the raw images. Therefore, for MPC and PD, the initial converged grasping forces are actually larger than the proper grasping force.}}
\label{fig:mpc_pd_shaking}
\end{figure*}

\begin{figure}[ht]
\centering
\begin{overpic}[trim=85 80 75 65,clip, width=0.48\textwidth]{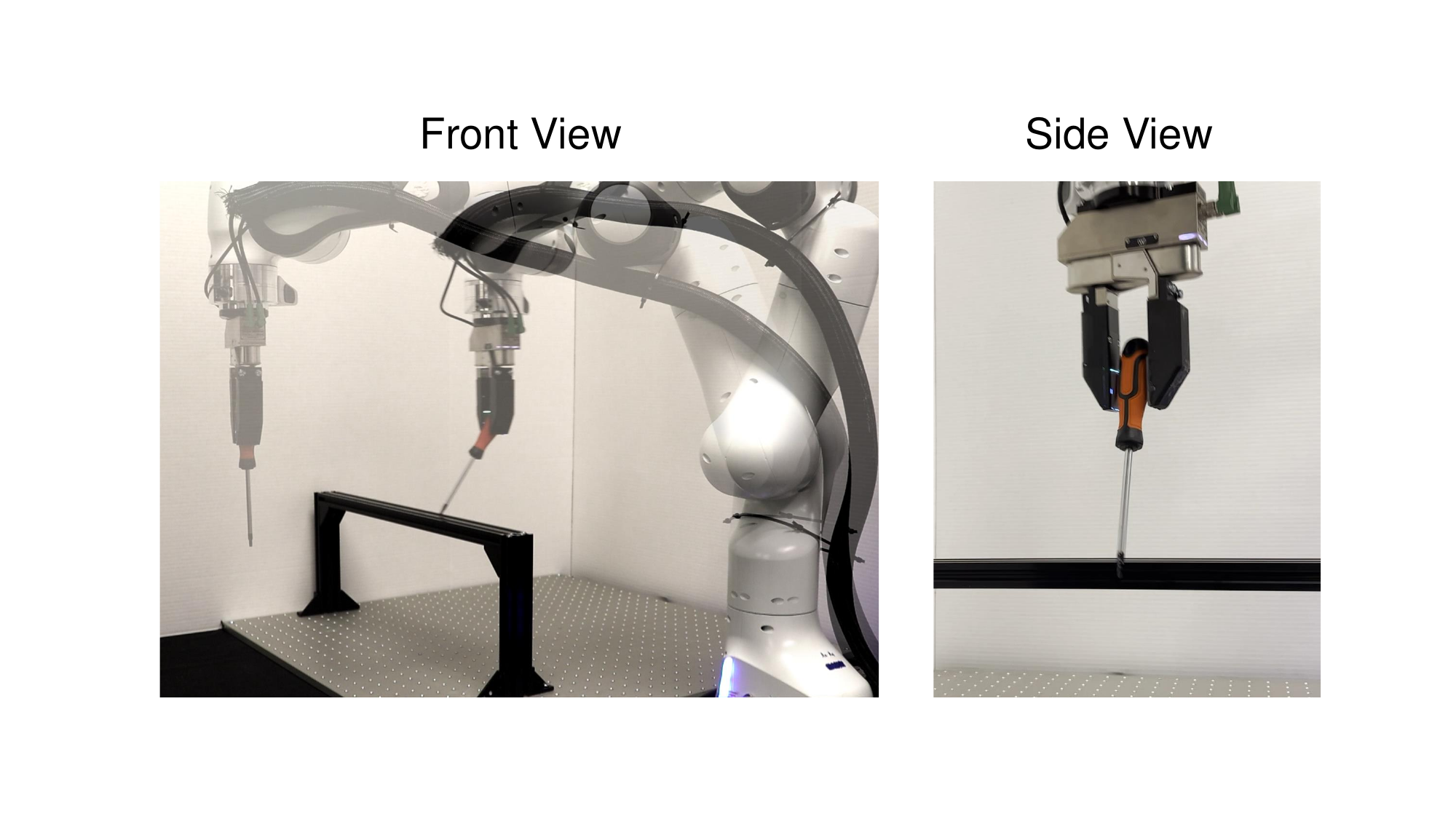}
\end{overpic}
\caption{Illustration of the obstacle collision experiment.}
\label{fig:ex3}
\end{figure}

\begin{figure*}[ht]
\centering
\begin{overpic}[trim=0 0 0 0,clip, width=0.99\textwidth]{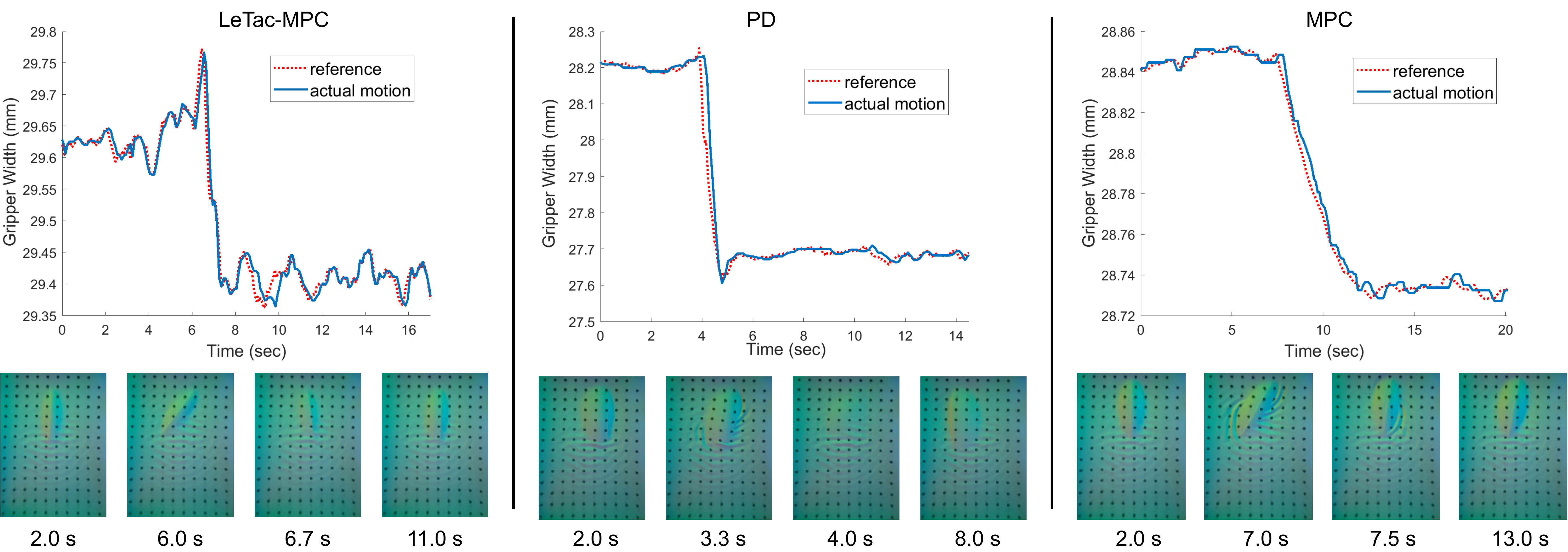}
\end{overpic}
\caption{Gripper width curves of the obstacle collision experiment.}
\label{fig:collision}
\end{figure*}

\begin{figure*}[ht]
\centering
\begin{overpic}[trim=0 0 0 0,clip, width=0.75\textwidth]{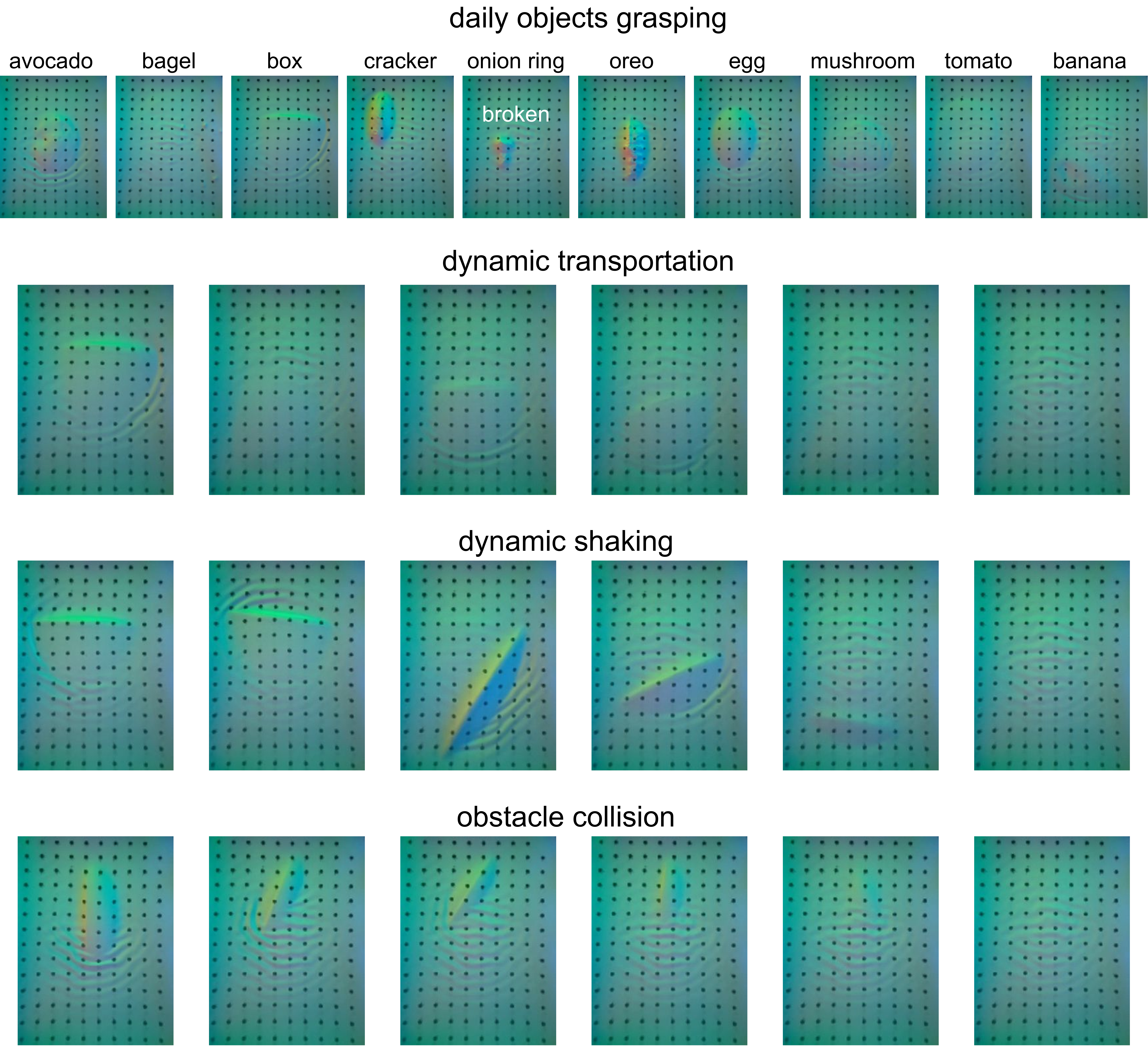}
\end{overpic}
\caption{Tactile image visualizations of open-loop grasping.}
\label{fig:openloop}
\end{figure*}

Firstly, we propose such linear assumption: 
\begin{align}
\label{eq:linear_assum}
c_{n+1} = c_n - K_c v_n \Delta t,
\end{align}
where $K_c$ is a scalar factor. Equation \eqref{eq:linear_assum} assumes a linear relationship between the contact area $c$ and the gripper width $p$, which is a simple and valid approximation of complex gel dynamics. The linear model is a local approximation for a very short duration. In MPC, the model is not used for long-term predictions, as this would lead to significant errors. Our proposed MPC baseline method runs at a high frequency with receding horizon control, ensuring that we can continuously iterate with the local model to minimize model errors. Combining equations \eqref{eq:gripper_model} and \eqref{eq:linear_assum}, we get the following model:

\begin{align}
\label{eq:mpc_model}
\left[\begin{array}{c}
            c_{n+1} \\
            p_{n+1}\\
            v_{n+1}   
            \end{array}\right] = \left[\begin{array}{cccc}
            1 & 0  &- K_c \Delta t\\    
            0  & 1 & \Delta t   \\
             0 & 0 & 1 
            \end{array}\right]
\left[\begin{array}{c}
            c_n \\
            p_n\\
            v_n   
            \end{array}\right]
+\left[\begin{array}{c}
            0 \\
            \frac{1}{2}\Delta t^2\\
            \Delta t   
            \end{array}\right]a_n.
\end{align}

Defining a feedback state vector $\mathbf{y}_n = [c_n,-d_n,v_n]^T$. and the length of prediction $N$, we write the same control objectives with PD as following cost function:
 
\begin{align}
J(\mathbf{y}_n,\mathbf{a}_n)
& = P\mathbf{e}_{n+N}^T{\mathbf{Q}}\mathbf{e}_{n+N} + \sum_{k=n}^{n+N-1} \mathbf{e}_{k}^T{\mathbf{Q}}\mathbf{e}_{k} +Q_a a_{k}^2,\label{eq:cost_function}\\
\text{where}~\mathbf{Q} &= \left[\begin{array}{ccc}
            Q_c & Q_dQ_c & 0 \\
            Q_dQ_c & Q_d^2Q_c & 0 \\
            0 & 0 & Q_v 
        \end{array}\right],\label{eq:q_matrix}\\
            \mathbf{e}_{n} &= \mathbf{y}_{n} - [c_{\text{ref}},0,0]^T,\label{eq:equilibrium}\\
            \mathbf{a}_n &= [
a_n,a_{n+1},\ldots{},a_{n+N-1}
]^T\in\mathbb{R}^{N}\notag.
\end{align}
In equation \eqref{eq:cost_function}, scalar $P$ is to amplify the terminal cost to speed up convergence. $Q_c, Q_v,~\text{and}~Q_a$ are weight coefficients. Equations \eqref{eq:q_matrix} and \eqref{eq:equilibrium} are derived from following control objective: $c_n $ converges to $ c_{\text{ref}}+Q_dd_n$ and $v_n$ converges to $0$.

Finally, we define a MPC law by solving following optimization problem:
\begin{align}
        \mathbf{a}_n^*&=\arg \min _{\mathbf{a}_n} J(\mathbf{y}_n,\mathbf{a}_{n}),\label{eq:opt_mpc}\\
        &\text{subject to}~\eqref{eq:mpc_model}~\text{and}\notag\\
        &\left[\begin{array}{c}
            p_{\text{min}} \\
           v_{\text{min}} \\
            a_{\text{min}}
            \end{array}\right]\leq
         \left[\begin{array}{c}
            p_n \\
            v_n \\
            a_n
            \end{array}\right]\leq
        \left[\begin{array}{c}
            p_{\text{max}} \\
           v_{\text{max}} \\
            a_{\text{max}}
        \end{array}\right].\notag
\end{align}

The optimization problem~\eqref{eq:opt_mpc} is a QP. Unlike the differentiable MPC layer, we do not need to do backpropagation for the optimization problem~\eqref{eq:opt_mpc}. Here, we only need to solve the optimization problem~\eqref{eq:opt_mpc} to compute the control inputs. Therefore, we use OSQP \cite{osqp}, an efficient QP solver to solve the optimization problem~\eqref{eq:opt_mpc} in real-time. The parameters we use for MPC are shown in Tables~\ref{tab:mpc}~and~\ref{tab:values_cons}.

\subsection{Open-loop Grasping}
For open-loop grasping, we use the force feedback of the WSG 50-110 gripper as the signal to choose the gripper width. We select 10~N as the threshold value. At the beginning of the grasping, the gripper width decreases, and once the force feedback is greater than 10~N, the gripper stops. Then the gripper width will remain in this position during the period of performing the task.

{We select 10~N as the threshold value after testing and adjusting it through grasping and dynamic transportation of the objects shown in Fig.~\ref{fig:daily_object} (see Section~\ref{sec:grasping} for more details).  10~N is suitable for most objects, as increasing the threshold could cause damage to fragile objects, while decreasing it may result in heavy objects dropping during dynamic transportation.}

It is important to note that even when we use 10~N as the static condition for the fingers in open-loop grasping, it does not imply that open-loop grasping can maintain a constant grasping force of 10~N. This variability is in part due to the inherent noise in the force feedback of the WSG gripper. Furthermore, when the force feedback reaches 10~N, the fingers are in motion, and the force exerted differs once they stop. Additionally, the final resting force varies with the shape and properties of the grasped object, as shown in Table~\ref{tab:force}.

\section{Experiments}\label{sec:exp}
We conduct experimental validation of proposed LeTac-MPC on a WSG 50-110 parallel gripper and a Franka Panda manipulator, as shown in Fig.~\ref{fig:firstPage}. We run LeTac-MPC at a frequency of 25~Hz. Using the official WSG 50 ROS package\footnote{https://github.com/nalt/wsg50-ros-pkg}, the maximum rate of the gripper is 30~Hz. Therefore, the gripper can track the motion generated by LeTac-MPC with reasonable saturation constraints. We perform three tasks, namely daily objects grasping and transportation (Section~\ref{sec:grasping}), dynamic shaking (Section~\ref{sec:shaking}), and obstacle collision (Section~\ref{sec:obs}), and compare LeTac-MPC with PD control, MPC, and open-loop grasping. To capture the tactile feedback, we mount a GelSight on one of the fingers of the WSG 50-110 gripper. A video includes these three tasks and comparisons with the baseline methods can be found in the supplemental material. 

The WSG gripper we use has a built-in low-level controller, which can track position reference.  Our designed reactive grasping controllers are high-level controllers. These high-level controllers produce reference motions for the gripper fingers with reasonable saturation constraints, which are then accurately tracked by the low-level controller. In the case of LeTac-MPC and MPC, the solution to the optimization problem is acceleration $a_n$, but our state also includes gripper width $p_n$ and velocity $v_n$. Eventually, we use $p_n$ from the state as the reference sent to the low-level controller for tracking. Although we do not directly track $v_n$ and $a_n$, since $p$ is derived from the integration of $v$ and $a$, what the low-level controller effectively tracks is the entire trajectory of the finger motion. As for the PD controller, at each moment, we calculate $v_n$ through feedback, which, after integration, gives us the position reference $p_n$ to be tracked by the low-level controller. Hence, the low-level controller actually tracks both $p$ and $v$.

\subsection{Daily Objects Grasping and Transportation}\label{sec:grasping}
We evaluate our proposed LeTac-MPC method and compare it with three baseline methods for a daily objects grasping and transportation task. As shown in Fig.~\ref{fig:daily_object}, we select 10 daily objects. These objects differ in physical properties (such as stiffness, total weight, mass distribution, and coefficient of friction), sizes and shapes (ranging from larger to smaller objects), and surface textures (from smooth surfaces to different complex textures). Some objects are also fragile (e.g. onion ring, egg, and cracker). All of these present challenges to the generalization and robustness of grasping. The position of dynamic transportation is shown in Fig.~\ref{fig:traj}(a). We use this trajectory for evaluating all objects and methods. The only exception is the box containing beads, which undergoes a relative displacement along the Z-axis to accommodate its larger size. Objects are handed to the gripper by someone to facilitate grasping. Among evaluating for different methods, each object will be grasped with the same configuration. 

The task involves grasping different daily objects, each with different sizes, shapes, physical properties, and surface textures, and transporting them to another location. The task presents three main challenges. First, the gripper must apply the appropriate force to grasp the object without damaging it but with enough force to hold it stably. Second, during the dynamic transportation of the object, the gripper needs to reactively adjust its behavior to maintain the object stably, particularly for heavier objects and objects with varying mass distribution. Third, the task requires generalizing feasible grasping to different objects with different physical properties, shapes, sizes, and surface textures.

Based on the experimental results, our proposed LeTac-MPC overcomes these challenges and outperforms the baseline methods. More detailed explanations of the experimental results are provided below:

1) The detailed results of the implementation of LeTac-MPC in banana and box containing beads are shown in Fig.~\ref{fig:banana}. When the object was not in contact with the gripper, the gripper width decreased due to the empty tactile feedback. When the object came into contact with the gripper, the gripper width gradually converged to a position based on the tactile feedback from the GelSight. After the controller converged, the manipulator dynamically transported the object to another location, and the tactile feedback changed due to rapid changes in the states the grasped object. Correspondingly, LeTac-MPC would enable the gripper to reactively re-grasp the object to stabilize it in hand. Finally, after the object was transported to another location and became static, the gripper width converged. 

Note that the contraction response of the gripper in the box experiment is faster than that of the banana. This is because the inertia of the box containing beads changes significantly during dynamic transportation, requiring a faster response to maintain its stability.

2) Fig.~\ref{fig:daily_grasping} shows the tactile images of all 10 daily objects in the experiments. Table~\ref{tab:force} shows the results and mean grasping forces of the experiments. Our proposed LeTac-MPC is able to generalize to different objects with different physical properties, sizes, shapes, and surface textures, as shown in Fig.~\ref{fig:daily_grasping}. {Items such as cracker (Fig.~\ref{fig:daily_grasping}(d)), onion ring (Fig.~\ref{fig:daily_grasping}(e)), and oreo (Fig.~\ref{fig:daily_grasping}(f)) are smaller than the size of the gripper and also have complex textures and cluttered contours.} The training dataset only includes four standardized blocks with different physical properties. This indicates a good generalizability of LeTac-MPC. In Fig.~\ref{fig:grasping_curve}, we show the gripper width curves of several experiments.

3) As shown in Figs.~\ref{fig:daily_grasping}(b)~and~\ref{fig:daily_grasping}(i), although the tactile feedback for soft objects such as bagel and tomato is subtle, LeTac-MPC can still grasp robustly. However, both PD and MPC controllers rely on the contact area as the feedback signal to compute control inputs, and therefore cannot function on soft objects that have subtle tactile feedback. This is because it is difficult to obtain a high-quality and stable contact area from tactile feedback for soft objects. There are two popular ways to extract the contact area mask of the tactile image, thresholding depth images (Fig.~\ref{fig:thresholding}) and difference images (Fig.~\ref{fig:diff_thresholding}). The depth image is the 3D reconstruction of the contact surface and can be obtained by the method in \cite{wang2021gelsight}. The difference image is obtained by subtracting the current frame image from the first frame image, with marker removal achieved by interpolation at the marker locations on both the current and the first frames.

We collect multiple tactile images for different objects in the experiment and present a sequence of results obtained by thresholding depth images and difference images by different values. As seen in Figs.~\ref{fig:thresholding}~and~\ref{fig:diff_thresholding}, the same thresholding value cannot be generalized to different objects with different physical properties. When the value is small, we cannot obtain a reasonable contact area for deformable objects. However, when the value is large, the contact area of rigid objects becomes larger than the reasonable area and has more noises. Furthermore, for bagel, regardless of how we adjust the value, the contact area always appears poor for computing control input. This is because the bagel is very soft, resulting in very subtle tactile feedback. The raw images are from MPC experimental records. This explains the results in Table~\ref{tab:force}. For PD, since it utilizes the same feedback signals and control objectives as MPC, this can also explain the experimental results of PD.

In our implementation, we use the contact area from depth images as the feedback signal for MPC and PD. Additionally, 0.05 is the threshold that we use. {As shown in Fig.~\ref{fig:thresholding}, 0.05 is the optimal threshold for rigid objects with textures and without textures based on our experiments in Fig.~\ref{fig:thresholding}.} {In Fig.~\ref{fig:thresholding}, we can see that using 0.025 as the threshold introduces noise in non-contact areas for hard objects such as cracker and avocado (visible on the right side), and the contact area for cracker becomes elongated, which does not match the actual ground truth. Although tomato shows improved contact area quality at 0.025, the quality for bagel remains poor. On the other hand, a 0.05 threshold performs better for rigid objects (cracker and avocado). Therefore, we chose 0.05, which generally works for rigid objects. }

The poor contact area signal leads to poor grasping behavior. As shown in Fig.~\ref{fig:grasping_curve}, MPC and PD for soft object grasping continuously decreased the gripper width until it reached the force limit of the WSG 50-110 gripper, as the large force applied by the gripper cannot create a sufficient contact area mask. However, LeTac-MPC performs exceptionally well in grasping deformable objects, even when tactile features are subtle.  

Note that for tomato grasping in Fig.~\ref{fig:grasping_curve}, MPC succeeded while PD did not. This does not necessarily mean that MPC outperformed PD in grasping this object. As shown in Fig.~\ref{fig:thresholding}, the contact area in the MPC experiment is also very unsatisfactory, which results in MPC convergence, but with a large gripping force applied. Both PD and MPC calculate control inputs based on the contact area, which makes both unsuitable for grasping deformable objects. The difference in the experimental results for grasping tomato can be attributed to other minor factors, such as the grasping position of the tomato.

4) {Table~\ref{tab:force} shows that LeTac-MPC demonstrates statistically smaller grasping forces for most objects and performs the optimal on the grasping and transportation task. For very soft objects, such as tomato and bagel, PD and MPC tend to apply excessive grasping forces. Often, these forces reach the gripper's force limit.} For open-loop grasping with large grasping force, it still fails to grasp the box containing beads stably during dynamic transportation. This is because the simple open-loop policy cannot re-grasp properly when the grasped object's state changes suddenly. We can see the tactile image sequence of dropping the box in Fig.~\ref{fig:openloop}. {Moreover, for certain objects in which the force applied during open-loop grasping is much greater than that in LeTac-MPC, we can see differences between tactile images Figs.~\ref{fig:daily_grasping}~and~\ref{fig:openloop}. For example, by comparing Figs.~\ref{fig:daily_grasping}~and~\ref{fig:openloop}, it is evident that the tactile image of the egg in Fig.~\ref{fig:openloop} has a larger contact area and clearer features, indicating that the force applied in open-loop grasping is greater.}

5) In our implementation, due to the slow convergence rate of PD and MPC, we set the gripper to contract at a constant speed of 2.5~mm/s when there is no contact with the object. When the gripper is in contact with the object, we switch to the PD and MPC controllers. As shown in Fig.~\ref{fig:grasping_curve}, the gripper width curves generated by PD and MPC both have a turning point (green circle), which represents the gripper width at the point of contact with the object. We can observe that after the turning point, PD and MPC gradually converge, but their convergence rate is slow. For LeTac-MPC, we do not set a constant contraction speed when the gripper is not in contact with an object. Instead, we apply LeTac-MPC to generate gripper motion throughout the grasping process. This is because LeTac-MPC has a faster convergence rate than PD and MPC, no matter if the gripper is in contact with the object, as shown in Fig.~\ref{fig:grasping_curve}.

6) We can observe that the fluctuation after convergence of PD and MPC is smaller than that of LeTac-MPC. This is because the output of the NN has some randomness, but the output of the model-based controllers is more stable. However, as shown in Figs.~\ref{fig:banana}~and~\ref{fig:grasping_curve}, the randomness in the output of our proposed network model does not affect grasping performance.

\subsection{Extreme Case: Grasping a Peeled Boiled Egg}

In this section, we use LeTac-MPC to grasp a peeled boiled egg, which is extremely soft and delicate. In fact, this experiment is not successful, but shows an extreme case of LeTac-MPC.

The experimental visualizations and gripper width curves are shown in Fig.~\ref{fig:egg_fail}.  We can see that, during the entire grasping process, the changes in tactile images are almost imperceptible to the human eye. This is because comparing with the elastomer of GelSight, the peeled boiled egg is much more softer. So, no matter how much grasping force is applied to the peeled biled egg, it cannot make enough deformation on the elastomer of GelSight.

When the elastomer touches the egg for the first time, there are some subtle features that show up in the tactile image. Therefore, the speed at which the gripper width decreases begins to vary, and it appears that the controller is trending toward convergence (point B in Fig.~\ref{fig:egg_fail}). However, it is important to note that the controller does not actually converge; instead, the gripper width decreases very slowly with a minimal slope (point B in Fig.~\ref{fig:egg_fail}). This is because, for grasping other daily objects in Fig.~\ref{fig:daily_object}, when the tactile image has features, LeTac-MPC will slowly decrease the gripper width and tend toward convergence. During this slow decrease, as the grasping force increases, the features of the tactile image become more pronounced, leading LeTac-MPC to gradually converge. However, as we mentioned, since a peeled boiled egg is too soft, no matter how much force is applied to it, its tactile features do not become stronger. Therefore, the gripper width continues to decrease at a very slow rate, and LeTac-MPC never converges. This ultimately leads to the egg being damaged (points C and D in Fig.~\ref{fig:egg_fail}).

\subsection{Dynamic Shaking}\label{sec:shaking}

In this experiment, we have the gripper grasp a box containing beads and shake it violently to explore how LeTac-MPC and baseline methods behave to prevent the box from falling under dynamic shaking. During dynamic shaking, the beads inside also shake, resulting in unpredictable state changes in the grasped box. The position of the end-effector during dynamic shaking is shown in Fig.~\ref{fig:traj}(b). The maximum acceleration can reach to 8.15~m/$\text{s}^2$. This makes the task very challenging.

Fig.~\ref{fig:nn_shaking} shows the results of LeTac-MPC. We can see that during the shaking process, the box shook in the hand due to rapid changes in its acceleration and inertia. Correspondingly, the gripper re-grasped based on the tactile feedback and eventually converged. Throughout the entire process, the box did not drop even though it shook in the hand. The focus of the experiment result is on whether the box falls or not. According to Fig.~\ref{fig:nn_shaking}, due to the rapid change in the state of the box, minor slips are common and we cannot control the slip distance. However, due to the rapid response of LeTac-MPC, it could quickly re-grasp the object and converge to a new grasping width to prevent it from falling.

In contrast, for open-loop control, we observe that the box easily drops because there is no reactive behavior in open-loop grasping, as shown in Fig.~\ref{fig:openloop}.

For MPC and PD, we observe that they have similar results, as shown in Fig.~\ref{fig:mpc_pd_shaking}. We can see that, at the beginning of the experiments, the contact areas are too small compared to the raw images. Therefore, for MPC and PD, the initial convergent grasping forces are actually larger than the proper grasping force. We can also obtain this conclusion by comparing the tactile images of PD and MPC in Fig.~\ref{fig:mpc_pd_shaking} with the tactile images of LeTac-MPC in Fig.~\ref{fig:nn_shaking}. In the shaking process, the gripper width first decreased due to the sudden change in the state of the grasped object, as seen in Fig.~\ref{fig:mpc_pd_shaking}. Then, for both PD and MPC, the gripper width increased and converged to a new larger position value, indicating that the initial grasping forces are too large. Since the performance of model-based control methods like PD and MPC rely on the quality of tactile feedback, their generalizability is lower than that of LeTac-MPC.

We conducted 10 repetitions of the experiment for each method and calculated the success rate, as shown in Table~\ref{tab:rate}.  We can see that the success rate of the reactive control methods (LeTac-MPC, MPC, PD) is significantly higher than that of the open-loop grasping. This is because PD and MPC tend to apply greater force due to the low-quality contact area feedback for the box, as shown in Fig.~\ref{fig:mpc_pd_shaking}. Additionally, by comparing the tactile images in Figs.~\ref{fig:nn_shaking}~and~\ref{fig:mpc_pd_shaking}, it can be observed that while PD and MPC achieve higher success rates, LeTac-MPC maintains a relatively high success rate while applying significantly less force.

In the daily objects grasping and transportation experiment, we demonstrate that LeTac-MPC's ability to reactively grasping to a variety of daily objects, since transportation is also a dynamic task. Additionally, this experiment also showcases both PD and MPC exhibit poor generalization across different objects. If a range of diverse objects were introduced in the dynamic shaking experiment, PD and MPC would struggle to grasp some of them effectively, particularly the soft ones. On the other hand, open-loop grasping lacks the capability for reactive grasping. In general, we anticipate that LeTac-MPC will outperform these baseline methods in such scenario. This can also be analogously applied to the obstacle collision experiment.

In Fig.~\ref{fig:nn_shaking}, there is an increase to the gripper width. When shaking occurs, the state and in-hand configuration of the object may change. In such cases, the force exerted on GelSight may increase, resulting in changes in the tactile image. Consequently, the controller will increase the gripper width to mitigate this impact. This also could happen in dynamic transportation and obstacle collision experiments. However, the state of the grasped object is constantly changing. If the force exerted on GelSight decreases, the LeTac-MPC will reduce the gripper width to ensure stable grasping. The entire process is dynamic and based on tactile image feedback. Consequently, we can observe that the gripper width is sharply decreased and converges to a new position.

\subsection{Obstacle Collision}\label{sec:obs}

In this section, we present the results of the experiment that tests the effect of collisions on LeTac-MPC and the baseline methods. In this experiment, we have the robot grasp a screwdriver and then collide with an obstacle, as shown in Fig.~\ref{fig:ex3}. This unexpected collision may occur during some robot manipulation tasks. The ideal grasping controller should be able to resist this type of external impact.

In this experiment, we chose a screwdriver as the grasped object, because it is rigid, has a regular shape, and lacks surface texture, all of which make it easily grasped by the three baseline methods (PD, MPC, and open-loop). Otherwise, as demonstrated in the daily objects grasping and transportation experiment, many objects are challenging for PD and MPC to grasp properly and robustly. Using this grasp-feasible object, we further focus on assessing the impact of obstacle collision.

As shown in Fig.~\ref{fig:collision}, in the LeTac-MPC, PD, and MPC experiments, the screwdriver became loose in the hand due to the collision, resulting in the gripper rapidly contracting to re-grasp the screwdriver. Depending on the specific grasping and collision situation, reactive grasping controllers re-grasp with corresponding contracting distances and velocities. Therefore, all three methods demonstrate reactive behavior in response to this type of external impacts. On the contrary, as shown in Fig.~\ref{fig:openloop}, open-loop grasping cannot handle this type of external impact. When a collision occurs, the grasped object always easily drops. We conducted 10 repetitions of the experiment for each method and calculated the success rate, as shown in Table~\ref{tab:rate}.  We can see that the success rate of the reactive control methods (LeTac-MPC, MPC, PD) is significantly higher than that of the open-loop grasping.

It is worth noting, as mentioned previously, that LeTac-MPC exhibits higher fluctuations after convergence compared to PD and MPC. However, these fluctuations are acceptable. As shown in Figs.~\ref{fig:nn_shaking}~and~\ref{fig:collision}, the fluctuation range is around $\pm 0.05$~mm. As demonstrated in the supplementary video,  under such minor fluctuations, changes in the tactile image are virtually indiscernible. Conversely, during dynamic shaking or when the controller reconverges after the grasped object collides with an external obstacle, the changes in the tactile image become quite visible. The differences between these two indicate that within the range of $\pm0.05$ mm, the grasp state remains almost constant.

Additionally, the LeTac-MPC is more sensitive to changes in tactile feedback and has a faster response speed. As shown in Fig.~\ref{fig:collision}, at the onset of collision, the tactile features became more pronounced, resulting in stronger feedback. At this point, the LeTac-MPC released the gripper a little bit to maintain a reasonable grasping force. However, with the change of tactile feedback, the gripper quickly contracted to stabilize the screwdriver.

\section{Discussion and Future work}

In this paper,  we propose LeTac-MPC, a learning-based model predictive control for tactile-reactive grasping to address the challenges of robotic tactile-reactive grasping for objects with different physical properties and dynamic and force-interactive tasks. The proposed approach features a differentiable MPC layer to model the embeddings extracted by the NN from the tactile feedback, enabling robust reactive grasping of different daily objects. We also present a fully automated data collection pipeline and demonstrate that our trained controller can adapt to various daily objects with different physical properties, sizes, shapes, and surface textures, although the controller is trained with 4 standardized blocks.  We perform three tasks: daily objects grasping and transportation, dynamic shaking, and obstacle collision, and compare 4 methods: LeTac-MPC, MPC, PD, and open-loop grasping. Through experimental comparisons, we show that LeTac-MPC has optimal performance on dynamic and force-interactive tasks and optimal generalizability. 

The fundamental idea of our LeTac-MPC lies in learning a unified representation and employing the differentiable MPC layer to utilize it. This approach enables generalization to a variety of daily objects, even when trained on a simplistic dataset. As clearly demonstrated in Fig.~\ref{fig:daily_grasping}, the tactile images of each testing object show a significant divergence from the training objects. This ability to generalize to unseen objects is a challenge for other learning-based grasping methods, particularly those based on classification, to achieve. It should be noted that LeTac-MPC does not explicitly output various physical properties of the grasped object (such as stiffness). Instead, it implicitly extracts these grasp-related properties and then utilizes these implicit features through the MPC layer to generate grasping actions.

\begin{table}[t]
\centering
\begin{tabular}{lllll}
\hline
                   & LeTac-MPC & MPC   & PD    & open-loop \\ \hline
dynamic shaking    & 8/10      & 10/10 & 10/10 & 2/10      \\
obstacle collision & 10/10     & 10/10 & 10/10 & 3/10      \\ \hline
\end{tabular}
\caption{Success rate of the dynamic shaking experiment and the obstacle collision experiment. }
\label{tab:rate}
\end{table}

For slippery object, as long as the object is not extremely slippery, LeTac-MPC can effectively achieve reactive grasping. In our experiments, items such as a box containing beads, a tomato, and an egg, which are all slippery, are successfully addressed by LeTac-MPC. However, this is based on the assumption that the surface of the grasped object is not extremely slippery. In such cases, the tangential force can still cause deformation of the gel surface, which can be captured by the camera in GelSight. If the surface of the object is extremely slippery and the application of sufficient grasping force does not cause enough tangential deformation of the gel surface, the algorithm lacks the necessary information to continue tightening the gripper, leading to a failure of grasping.

It should be noted that if we utilize a sensor elastomer with lower stiffness, we anticipate that LeTac-MPC, PD, and MPC will demonstrate enhanced performance in grasping soft objects. They are likely to use less force while achieving more robust grasping, attributed to the greater deformation of the elastomer upon contact with objects. However, a potential issue arises: the excessive softness of the elastomer might cause the grasped object to shake in the hand, which is undesirable for many manipulation tasks. Therefore, selecting an elastomer with the appropriate stiffness is crucial.

Our proposed LeTac-MPC has several limitations that could be potential research directions for future work. 

1. First of all, our method cannot grasp very soft objects, such as various types of meat. Due to the much lower stiffness of these objects compared to the GelSight sensor elastomer, the tactile feedback is very subtle and even difficult for human eyes to recognize. To grasp such soft objects, it may be necessary to design new grippers that are more suitable for soft objects and develop sensing and control algorithms that are compatible with new grippers to solve this problem at the system level. 

2. Second, our algorithm applies to grasping objects with different physical properties, but does not involve specific manipulation operations. In the future, we can explore how to perform complex manipulation tasks on objects with different physical properties based on tactile feedback. 

3. In addition, our work only utilizes tactile feedback as input to the network. However, for more complex manipulation tasks, combining visual and tactile feedback can provide richer information. In such cases, the linear differentiable MPC layer may not have sufficient ability to represent task objectives. To design algorithms that can generalize to various manipulation tasks, incorporating nonlinear optimization layers with visual and tactile input could be a potential research direction. 

4. {Vision-based tactile sensors, like GelSight, provide high-dimensional tactile feedback. This high-dimensional feedback contains a wealth of contact information. However, processing such high-dimensional data can challenge the frame rate of the entire learning-based control loop, potentially limiting its performance in certain tasks. Tactile sensors based on electrical signals, known for their good sensitivity and wide bandwidth, offer significant advantages. Exploring learning-based reactive grasping that combines tactile sensors based on electrical signals and vision is a promising direction.} 

5. Finally, since the raw image of different GelSight sensors looks slightly different, a model trained with data from one sensor might perform less effectively when transferred to another sensor. However, the imaging principle of GelSight is the same and the images from different sensors follow certain patterns and have commonalities. Therefore, training a controller or policy on one sensor and generalizing it to other GelSights presents a promising future research direction.

\section{acknowledgements}
The authors thank Yipai Du for helpful discussions about tactile sensing. The authors thank Raghava Uppuluri for setting up the initial robot system. This work was partially supported by the National Science Foundation (NSF; No. 2423068). This article solely reflects the
opinions and conclusions of its authors and not of NSF.

\ifCLASSOPTIONcaptionsoff
  \newpage
\fi

\bibliographystyle{IEEEtran}
\bibliography{IEEEabrv,paperref}

\begin{IEEEbiography}[{\includegraphics[width=1in,height=1.25in,clip,keepaspectratio]{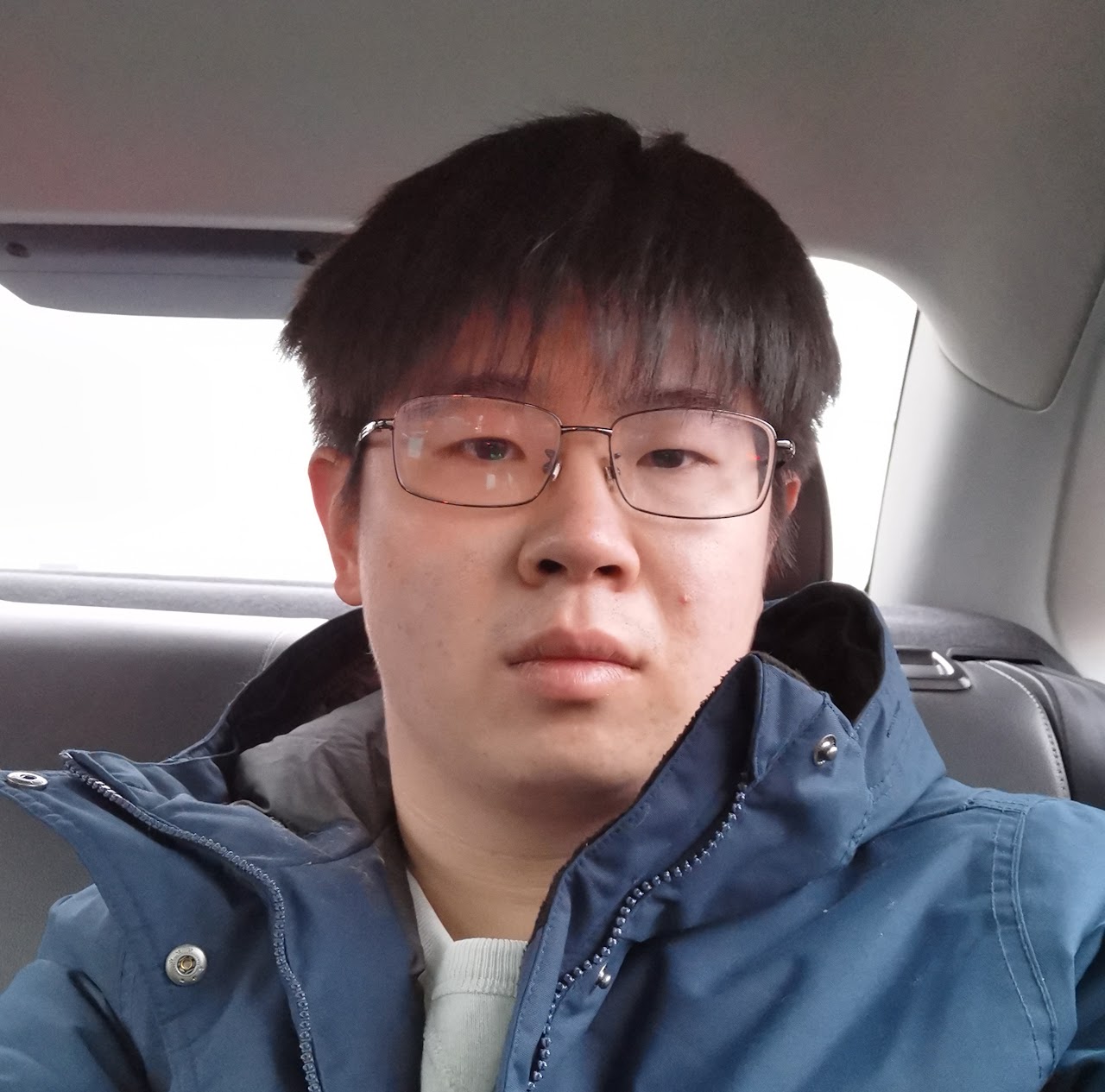}}]{Zhengtong Xu} received his Bachelor's degree in mechanical engineering from Huazhong University of Science and Technology, China. He is currently pursuing his Ph.D. at Purdue University. His research focus is on robot learning.
\end{IEEEbiography}
\vskip -2\baselineskip plus -1fil
\begin{IEEEbiography}[{\includegraphics[width=1in,height=1.25in,clip,keepaspectratio]{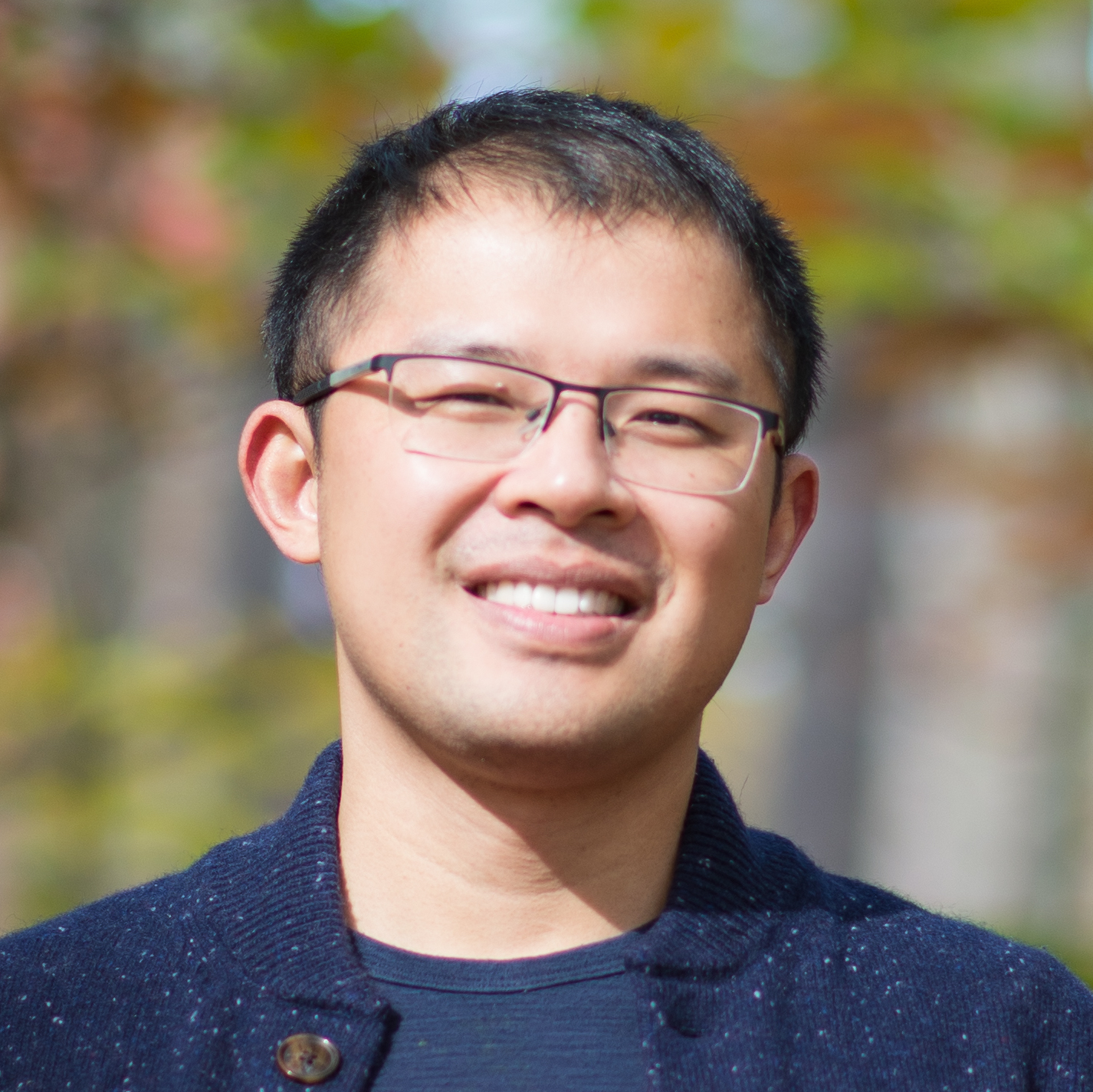}}]{Yu She} is an assistant professor at Purdue University School of Industrial Engineering. Prior to that, he was a postdoctoral researcher in the Computer Science and Artificial Intelligence Laboratory at MIT from 2018 to 2021. He earned his Ph.D. degree in the Department of Mechanical Engineering at the Ohio State University in 2018. His research, at the intersection of mechanical design, sensory perception, and dynamic control, explores human-safe collaborative robots, soft robotics, and robotic manipulation.
\end{IEEEbiography}

\end{document}